%% file: aistats2025.tex
\documentclass[twoside]{article}

\usepackage{microtype}
\usepackage{graphicx}
\usepackage{subfigure}
\usepackage[shortlabels,sizes,inline]{enumitem}
\setlist[enumerate]{
    label=\textit{\roman*}\hspace{0.05em}),
    nolistsep,
}

\input{packages}

\usepackage[accepted]{aistats2025}
\usepackage{amsmath}
\usepackage{amssymb}
\usepackage{mathtools}
\usepackage{siunitx}

\usepackage{booktabs, nicematrix} 

\everypar{\looseness=-1}

\theoremstyle{plain}
\newtheorem{theorem}{Theorem}[section]

\newtheorem{lemma}{Lemma}[section]
\newtheorem{corollary}{Corollary}[section]
\theoremstyle{definition}
\newtheorem{definition}{Definition}[section]

\theoremstyle{remark}
\newtheorem{remark}{Remark}[section]
\theoremstyle{plain}

\usepackage{cleveref}
\crefname{section}{\S}{\S\S}
\crefname{subsection}{\S}{\S\S}
\crefname{subsubsection}{\S}{\S\S}
\crefname{figure}{Fig.}{Figs.}
\crefname{prop}{Prop.}{Props.}
\crefname{appendix}{Appx.}{Appxs.}
\crefname{algorithm}{Alg.}{Algs.}
\crefname{theorem}{Thm.}{Thms.}
\crefname{definition}{Defn.}{Defns.}
\crefname{cor}{Corollary}{Corollaries}
\crefname{lem}{Lem.}{Lems.}
\crefname{table}{Tab.}{Tabs.}
\crefname{assum}{Assum.}{Assums.}
\crefname{example}{Ex.}{Exs.}

\theoremstyle{plain}
\newtheorem{reptheoreminner}{Theorem}
\newenvironment{reptheorem}[1]{%
  \renewcommand\thereptheoreminner{\identifierstyle{#1}}%
  \begin{reptheoreminner}
}{\end{reptheoreminner}}

\newtheorem{repcorollaryinner}{Corollary}
\newenvironment{repcorollary}[1]{%
  \renewcommand\therepcorollaryinner{\identifierstyle{#1}}%
  \begin{repcorollaryinner}
}{\end{repcorollaryinner}}

\theoremstyle{definition}
\newtheorem{repdefinitioninner}{Definition}

\input{commands}

\input{glossary}
\input{acronyms}
\usepackage{authblk}

\setglossarysection{section}

\begin{document}

%

%
\runningauthor{Rusak${}^{*}$, Reizinger${}^{*}$, Juhos${}^{*}$, Bringmann, Zimmermann${}^{\dagger}$, Brendel${}^{\dagger}$}

\twocolumn[

\aistatstitle{InfoNCE: Identifying the Gap Between Theory and Practice}

\aistatsauthor{
    Evgenia Rusak\textsuperscript{*,1,2} \And
    Patrik Reizinger\textsuperscript{*,2,4} \And
    Attila Juhos\textsuperscript{*,2,4} \And
    Oliver Bringmann\textsuperscript{1 }\AND
    Roland Zimmermann\textsuperscript{$\dagger$,2,4} \And
    Wieland Brendel\textsuperscript{$\dagger$,2,3,4}}

\vspace{5pt}

\aistatsaddress{\textsuperscript{1}University of Tübingen \quad \textsuperscript{2}Max Planck Institute for Intelligent Systems \\ \textsuperscript{3}ELLIS Institute Tübingen
\quad \textsuperscript{4}Tübingen AI Center}

]

\renewcommand*{\thefootnote}{\fnsymbol{footnote}}
\footnotetext[1]{Equal contribution. \footnotemark[2]{}Joint senior authors.}
\footnotetext[0]{Correspondence to \href{mailto:evgenia.rusak@uni-tuebingen.de}{\texttt{evgenia.rusak@uni-tuebingen.de}}\vspace{5pt}}

\footnotetext[0]{Code at \href{https://github.com/brendel-group/AnInfoNCE}{\texttt{github.com/brendel-group/AnInfoNCE}}.}
\footnotetext[0]{Website at \href{https://brendel-group.github.io/AnInfoNCE/}{\texttt{brendel-group.github.io/AnInfoNCE/}}.}

\begin{abstract}
 \input{abstract}
\end{abstract}

\input{main_text}


\bibliography{aistats2025}

\clearpage
\section*{Checklist}



 \begin{enumerate}

 \item For all models and algorithms presented, check if you include:
 \begin{enumerate}
   \item A clear description of the mathematical setting, assumptions, algorithm, and/or model. [Yes] Our theoretical assumptions are summarized in \cref{assum:non_isotropic_cl}, the experimental settings in \cref{sec:experiments}.
   \item An analysis of the properties and complexity (time, space, sample size) of any algorithm. [Yes] We provide a detailed analysis of the effect of batch size and computational requirements in \cref{app:analysis,app:compute}.
   \item (Optional) Anonymized source code, with specification of all dependencies, including external libraries. [Yes] We include anonymized code in the supplement. In addition, we described our software stack in \cref{app:software}.
 \end{enumerate}

 \item For any theoretical claim, check if you include:
 \begin{enumerate}
   \item Statements of the full set of assumptions of all theoretical results. [Yes] Confer \cref{assum:non_isotropic_cl} and \cref{sec:identifiability_proofs}.
   \item Complete proofs of all theoretical results. [Yes] Confer  \cref{sec:identifiability_proofs}
   \item Clear explanations of any assumptions. [Yes]  Confer \cref{sec:ident_theory}.
 \end{enumerate}

 \item For all figures and tables that present empirical results, check if you include:
 \begin{enumerate}
   \item The code, data, and instructions needed to reproduce the main experimental results (either in the supplemental material or as a URL). [Yes] We describe our experimental details in \cref{sec:experiments} and the software stack in \cref{app:software}.
   \item All the training details (e.g., data splits, hyperparameters, how they were chosen). [Yes]  We describe our experimental details in \cref{sec:experiments} 
         \item A clear definition of the specific measure or statistics and error bars (e.g., with respect to the random seed after running experiments multiple times). [Yes] Confer \cref{sec:experiments} 
         \item A description of the computing infrastructure used. (e.g., type of GPUs, internal cluster, or cloud provider). [Yes] We detail compute requirements in \cref{app:compute}.
 \end{enumerate}

 \item If you are using existing assets (e.g., code, data, models) or curating/releasing new assets, check if you include:
 \begin{enumerate}
   \item Citations of the creator If your work uses existing assets. [Yes] Confer \cref{sec:experiments} and \cref{app:software}.
   \item The license information of the assets, if applicable. [Not Applicabble]
   \item New assets either in the supplemental material or as a URL, if applicable. [No] We will release the code upon acceptance.
   \item Information about consent from data providers/curators. [Not Applicable]
   \item Discussion of sensible content if applicable, e.g., personally identifiable information or offensive content. [Not Applicable]
 \end{enumerate}

 \item If you used crowdsourcing or conducted research with human subjects, check if you include:
 \begin{enumerate}
   \item The full text of instructions given to participants and screenshots. [Not Applicable]
   \item Descriptions of potential participant risks, with links to Institutional Review Board (IRB) approvals if applicable. [Not Applicable]
   \item The estimated hourly wage paid to participants and the total amount spent on participant compensation. [Not Applicable]
 \end{enumerate}

 \end{enumerate}

\newpage
\appendix
\onecolumn

\input{appendix}

\end{document}

%% file: packages.tex
\usepackage[round]{natbib}

\usepackage[acronym, automake, toc, nomain, nopostdot, style=tree, nonumberlist,numberedsection]{glossaries}
\usepackage{glossary-mcols}
\setglossarysection{section}

\usepackage[utf8]{inputenc} 
\usepackage{microtype}      
\usepackage{pifont}

\usepackage[backref=page]{hyperref}       
\usepackage{url}            
\usepackage[capitalize,noabbrev,nameinlink]{cleveref}
\usepackage{crossreftools}

\crefname{section}{Sec.}{Secs.}
\crefname{subsection}{Sec.}{Secs.}
\crefname{subsubsection}{Sec.}{Secs.}
\crefname{figure}{Fig.}{Figs.}
\crefname{prop}{Prop.}{Props.}
\crefname{appendix}{Appx.}{Appxs.}
\crefname{algorithm}{Alg.}{Algs.}
\crefname{theorem}{Thm.}{Thms.}
\crefname{equation}{Eq.}{Eqs.}
\crefname{definition}{Defn.}{Defns.}
\crefname{def}{Defn.}{Defns.}
\crefname{cor}{Cor.}{Cors.}
\crefname{corollary}{Cor.}{Cors.}
\crefname{remark}{Rem.}{Rem.}
\crefname{lem}{Lem.}{Lems.}
\crefname{table}{Tab.}{Tabs.}
\crefname{assum}{Assum.}{Assums.}
\crefname{example}{Ex.}{Exs.}
\crefname{reptheorem}{Thm.}{Thms.}
\crefname{repdefinition}{Defn.}{Defns.}

\renewcommand*{\backref}[1]{}
\renewcommand*{\backrefalt}[4]{%
    \ifcase #1%
          \or Cited on page~#2.%
          \else Cited on pages~#2.%
    \fi%
    }

\usepackage{amsfonts}       
\usepackage{amsthm}
\usepackage{thmtools}
\usepackage{thm-restate} 
\usepackage{nicefrac}       
\usepackage{relsize}        

\usepackage{xfrac}

\usepackage{wrapfig}
\usepackage{graphicx}
\usepackage{caption}
\usepackage{subcaption}

\usepackage{tablefootnote}
\usepackage{booktabs}       
\usepackage{multirow}

\let\classAND\AND
\let\AND\relax
\usepackage{algorithmic}

\let\AND\classAND
\AtBeginEnvironment{algorithmic}{\let\AND\algoAND}

\usepackage{todonotes}

\usepackage{xcolor}         
\usepackage{xparse}
\usepackage{xspace}
\usepackage[warn-unused]{scpcmd}

\let\ORGhypersetup\hypersetup
\protected\def\hypersetup{\ORGhypersetup}
\usepackage{sidecap}
\sidecaptionvpos{figure}{c}

%% file: commands.tex
\newcommand{\clmarginal}{\ensuremath{\marginal{\gls{latent}}}\xspace}
\newcommand{\clconditional}{\ensuremath{\conditional{\gls{latentpos}}{\gls{latent}}}\xspace}

\newcommand{\expinner}[3]{\ensuremath{e^{\transpose{#1}\!\parenthesis{#2} {#1}\parenthesis{#3}/\gls{temp}}}}

\newcommand{\expinnerpos}{\expinner{\gls{enc}}{\gls{obs}}{\gls{obspos}}}
\newcommand{\expinnerneg}{\expinner{\gls{enc}}{\gls{obs}}{\obsnegi[i]}}

\newcommand{\ourloss}{\gls{aninfonce}\xspace}

\makeatletter
\newcommand*{\defeq}{\mathrel{\rlap{%
			\raisebox{0.3ex}{$\m@th\cdot$}}%
		\raisebox{-0.3ex}{$\m@th\cdot$}}%
	=}

\newcommand*{\eqdef}{=\mathrel{\rlap{%
			\raisebox{0.3ex}{$\m@th\cdot$}}%
		\raisebox{-0.3ex}{$\m@th\cdot$}}}
\makeatother

\newcommand{\ie}{i.e.\@\xspace}

\newcommand{\eg}{e.g.\@\xspace}

\newcommand{\cf}{cf.\@\xspace} 

\newcommand{\wrt}{w.r.t.\@\xspace} 

\newcommand{\inv}[1]{\ensuremath{#1^{-1}}}

\newcommand{\transpose}[1]{\ensuremath{#1^{\top}}}

\newcommand{\diag}[1]{\ensuremath{\mathrm{diag}\parenthesis{#1}}}

\newcommand{\mat}[1]{\ensuremath{\boldsymbol{\mathrm{#1}}}}

\newcommand{\myvec}[1]{\ensuremath{\boldsymbol{#1}}}

\newcommand{\norm}[1]{\ensuremath{\left\Vert#1\right\Vert}}
\newcommand{\normsquared}[1]{\ensuremath{\norm{#1}^2}}

\newcommand{\ltwo}{\ensuremath{\ell_{2}}\xspace}

\newcommand{\conditional}[2]{\ensuremath{p\parenthesis{#1|#2}}}
\newcommand{\marginal}[1]{\ensuremath{p\parenthesis{#1}}}

\NewDocumentCommand{\expectation}{o}{\ensuremath{\mathbb{E}\IfValueT{#1}{_{#1}}}}
\newcommand{\rr}[1][]{\ensuremath{\mathbb{R}^{#1}}}

\NewDocumentCommand{\dist}{o}{p\IfValueT{#1}{_{#1}}}
\newcommand{\prob}{\mathbb{P}}

\makeatletter
\newcommand{\raisemath}[1]{\mathpalette{\raisem@th{#1}}}
\newcommand{\raisem@th}[3]{\raisebox{#1}{$#2#3$}}
\makeatother

\newcommand{\parenthesis}[1]{\ensuremath{\left(#1\right)}}
\newcommand{\brackets}[1]{\ensuremath{\left[#1\right]}}
\newcommand{\braces}[1]{\ensuremath{\left\{#1\right\}}}

\DeclarePairedDelimiter{\scprod}{\langle}{\rangle}



%% file: glossary.tex
\newglossary{abbrev}{abs}{abo}{Nomenclature}

\newglossaryentry{r2}{
    name        = \ensuremath{R^2} ,
    description = {coefficient of determination} ,
    type        = abbrev,
}

\newglossaryentry{im}{
    name        = \ensuremath{\mathrm{Im}} ,
    description = {image space} ,
    type        = abbrev,
}

\newglossaryentry{ker}{
    name        = \ensuremath{\mathrm{Ker}} ,
    description = {kernel space} ,
    type        = abbrev,
}

\newglossaryentry{kronecker}{
    name        = \ensuremath{\otimes} ,
    description = {Kronecker product} ,
    type        = abbrev,
}

\newglossaryentry{loss}{
    name        = \ensuremath{\mathcal{L}} ,
    description = {loss function} ,
    type        = abbrev,
}

\newglossaryentry{hypersphere}{
    name        = \ensuremath{\mathbb{S}} ,
    description = {hypersphere} ,
    type        = abbrev,
}
\newcommand{\Sd}[1][d]{\ensuremath{\gls{hypersphere}^{#1}}}

\newglossaryentry{dec}{
    name        = \ensuremath{\mathrm{\boldsymbol{g}}} ,
    description = {decoder map $\gls{Latent}\to\gls{Obs}$} ,
    type        = abbrev,
}

\newglossaryentry{enc}{
    name        = \ensuremath{\mathrm{\boldsymbol{f}}} ,
    description = {encoder map $\gls{Obs}\to\gls{Latent}$} ,
    type        = abbrev,
}

\newglossaryentry{observations}{type=abbrev,name=Observations,description={\nopostdesc}}

\newglossaryentry{obs}{
    name        = \ensuremath{\boldsymbol{x}} ,
    description = {observation vector} ,
    type        = abbrev,
    parent      = observations,
}

\newglossaryentry{obscomp}{
    name        = \ensuremath{x} ,
    description = {observation single component} ,
    type        = abbrev,
    parent      = observations,
}

\newglossaryentry{Obs}{
    name        = \ensuremath{\mathcal{X}} ,
    description = {observation space} ,
    type        = abbrev,
    parent      = observations,
}

\newglossaryentry{obsdim}{
    name        = \ensuremath{D} ,
    description = {dimensionality of the observation space \gls{Obs}} ,
    type        = abbrev,
    parent      = observations,
}

\newglossaryentry{numneg}{
    name        = \ensuremath{M} ,
    description = {number of negative samples} ,
    type        = abbrev,
    parent      = observations,
}

\newglossaryentry{obspos}{
    name        = \ensuremath{{\boldsymbol{x}}^{\!+}} ,
    description = {positive observation vector} ,
    type        = abbrev,
    parent      = observations,
}

\newglossaryentry{obsneg}{
    name        = \ensuremath{\boldsymbol{x}^{\!-}} ,
    description = {negative observation vector} ,
    type        = abbrev,
    parent      = observations,
}

\newcommand{\obsnegi}[1][i]{\ensuremath{\myvec{x}^{\!-}_{#1}}}
\newcommand{\latentnegi}[1][i]{\ensuremath{\myvec{z}^{\!-}_{#1}}}

\newglossaryentry{labels}{type=abbrev,name=Labels,description={\nopostdesc}}

\newglossaryentry{label}{
    name        = \ensuremath{\boldsymbol{y}} ,
    description = {label vector} ,
    type        = abbrev,
    parent      = labels,
}

\newglossaryentry{labelhat}{
    name        = \ensuremath{\widehat{\boldsymbol{y}}} ,
    description = {estimated label vector} ,
    type        = abbrev,
    parent      = labels,
}

\newglossaryentry{labelcomp}{
    name        = \ensuremath{y} ,
    description = {label component} ,
    type        = abbrev,
    parent      = labels,
}

\newglossaryentry{labelcomphat}{
    name        = \ensuremath{\widehat{y}} ,
    description = {label component} ,
    type        = abbrev,
    parent      = labels,
}

\newglossaryentry{labelset}{
    name        = \ensuremath{\mathcal{Y}} ,
    description = {label set} ,
    type        = abbrev,
    parent      = labels,
}

\newglossaryentry{labeldim}{
    name        = \ensuremath{C} ,
    description = {number of classes in the label set \gls{labelset}} ,
    type        = abbrev,
    parent      = labels,
}

\newglossaryentry{latents}{type=abbrev,name=Latents,description={\nopostdesc}}

\newglossaryentry{latent}{
    name        = \ensuremath{\boldsymbol{z}} ,
    description = {latent vector} ,
    type        = abbrev,
    parent     = latents,
}

\newglossaryentry{latentrec}{
    name        = \ensuremath{\hat{\boldsymbol{z}}} ,
    description = {reconstructed latent vector} ,
    type        = abbrev,
    parent     = latents,
}

\newglossaryentry{latentcomp}{
    name        = \ensuremath{z} ,
    description = {latent single component} ,
    type        = abbrev,
    parent     = latents,
}

\newcommand{\latenti}[1][i]{\ensuremath{\gls{latentcomp}_{#1}}\xspace}

\newglossaryentry{Latent}{
    name        = \ensuremath{\mathcal{Z}} ,
    description = {latents} ,
    type        = abbrev,
    parent     = latents,
}

\newglossaryentry{latentdim}{
    name        = \ensuremath{d} ,
    description = {dimensionality of the latent space \gls{Latent}} ,
    type        = abbrev,
    parent     = latents,
}

\newglossaryentry{latentpos}{
    name        = \ensuremath{{\boldsymbol{z}}^{\!+}} ,
    description = {positive latent vector} ,
    type        = abbrev,
    parent      = latents,
}

\newglossaryentry{latentneg}{
    name        = \ensuremath{{\boldsymbol{z}}^{\!-}} ,
    description = {negative latent vector} ,
    type        = abbrev,
    parent      = latents,
}

\newglossaryentry{latentrecpos}{
    name        = \ensuremath{\hat{\boldsymbol{z}}^{\!+}} ,
    description = {reconstructed latent vector} ,
    type        = abbrev,
    parent     = latents,
}

\newglossaryentry{latentrecneg}{
    name        = \ensuremath{\hat{\boldsymbol{z}}^{\!-}} ,
    description = {reconstructed latent vector} ,
    type        = abbrev,
    parent     = latents,
}

\newglossaryentry{sigmaz}{
    name        = \ensuremath{\boldsymbol{\sigma}_{\gls{latentcomp}}} ,
    description = {std of \gls{latentcomp}} ,
    type        = abbrev,
    parent     = latents,
}

\newglossaryentry{content}{
    name        = \ensuremath{\boldsymbol{z}_{c}} ,
    description = {content latent vector} ,
    type        = abbrev,
    parent     = latents,
}

\newglossaryentry{contentcomp}{
    name        = \ensuremath{z_{c}} ,
    description = {content latent single component} ,
    type        = abbrev,
    parent     = latents,
}

\newglossaryentry{Content}{
    name        = \ensuremath{\mathcal{Z}_{c}} ,
    description = {content} ,
    type        = abbrev,
    parent     = latents,
}

\newglossaryentry{contentdim}{
    name        = \ensuremath{d_{c}} ,
    description = {dimensionality of \gls{content}} ,
    type        = abbrev,
    parent     = latents,
}

\newglossaryentry{sigmac}{
    name        = \ensuremath{\boldsymbol{\sigma}_{c}} ,
    description = {std of \gls{contentcomp}} ,
    type        = abbrev,
    parent     = latents,
}

\newcommand{\varc}{\ensuremath{\boldsymbol{\sigma}_{c}^2}\xspace}

\newglossaryentry{style}{
    name        = \ensuremath{\boldsymbol{z}_{s}} ,
    description = {style latent vector} ,
    type        = abbrev,
    parent     = latents,
}

\newglossaryentry{stylecomp}{
    name        = \ensuremath{z_{s}} ,
    description = {style latent single component} ,
    type        = abbrev,
    parent     = latents,
}

\newglossaryentry{Style}{
    name        = \ensuremath{\mathcal{Z}_{s}} ,
    description = {style} ,
    type        = abbrev,
    parent     = latents,
}

\newglossaryentry{styledim}{
    name        = \ensuremath{d_{s}} ,
    description = {dimensionality of \gls{style}} ,
    type        = abbrev,
    parent     = latents,
}

\newglossaryentry{sigmas}{
    name        = \ensuremath{\boldsymbol{\sigma}_{s}} ,
    description = {std of \gls{stylecomp}} ,
    type        = abbrev,
    parent     = latents,
}
\newcommand{\vars}{\ensuremath{\boldsymbol{\sigma}_{s}^2}\xspace}

\newglossaryentry{algebra}{type=abbrev,name=Algebra,description={\nopostdesc}}

\newglossaryentry{identity}{
    name        = \ensuremath{\boldsymbol{\mathrm{I}}_{\gls{obsdim}}} ,
    description = {\gls{obsdim}-dimensional identity matrix} ,
    type        = abbrev,
    parent      = algebra,
}

\newcommand{\Id}[1]{\ensuremath{\boldsymbol{\mathrm{I}}_{#1}}}

\newglossaryentry{jacobian}{
    name        = \ensuremath{\boldsymbol{\mathrm{J}}} ,
    description = {Jacobian matrix} ,
    type        = abbrev,
    parent      = algebra,
}

\newglossaryentry{hessian}{
    name        = \ensuremath{\boldsymbol{\mathrm{H}}} ,
    description = {Hessian matrix} ,
    type        = abbrev,
    parent      = algebra,
}

\newglossaryentry{d}{
    name        = \ensuremath{\boldsymbol{\mathrm{D}}} ,
    description = {diagonal matrix} ,
    type        = abbrev,
    parent      = algebra,
}

\newglossaryentry{diagtemp}{
    name        = \ensuremath{\boldsymbol{\mathrm{\Lambda}}} ,
    description = {diagonal matrix of weighting factors} ,
    type        = abbrev,
    parent      = algebra,
}

\newglossaryentry{diagtempinfer}{
    name        = \ensuremath{\boldsymbol{\hat{\mathrm{\Lambda}}}} ,
    description = {inferred diagonal matrix of weighting factors},
    type        = abbrev,
    parent      = algebra,
}

\newglossaryentry{o}{
    name        = \ensuremath{\boldsymbol{\mathrm{O}}},
    description = {orthogonal matrix} ,
    type        = abbrev,
    parent      = algebra,
}

\newglossaryentry{ones}{
    name        = \ensuremath{\mathbf{1}} ,
    description = {a column-vector with 1's} ,
    type        = abbrev,
    parent      = algebra,
}
\newglossaryentry{scalar}{
    name        = \ensuremath{\alpha} ,
    description = {scalar field} ,
    type        = abbrev,
    parent      = algebra,
}

\newglossaryentry{perm}{
    name        = \ensuremath{\mathbb{P}} ,
    description = {group of permutation matrices} ,
    type        = abbrev,
    parent      = algebra,
}

\newglossaryentry{permutation}{
    name        = \ensuremath{\mat{P}},
    description = {permutation matrix} ,
    type        = abbrev,
    parent      = algebra,
}

\newglossaryentry{prob}{type=abbrev,name=Probability theory,description={\nopostdesc}}

\newglossaryentry{cov}{
    name        = \ensuremath{\boldsymbol{\mathrm{\Sigma}}},
    description = {covariance matrix} ,
    type        = abbrev,
    parent      = prob,
}

\newglossaryentry{entropy}{
    name        = \ensuremath{\mathrm{H}} ,
    description = {entropy} ,
    type        = abbrev,
    parent      = prob,
}

\newglossaryentry{expfamparam}{
    name        = \ensuremath{\boldsymbol{\theta}} ,
    description = {parameter of exponential family} ,
    type        = abbrev,
    parent      = prob,
}

\newglossaryentry{expfamnatparam}{
    name        = \ensuremath{\boldsymbol{\eta}} ,
    description = {natural parameter of exponential family} ,
    type        = abbrev,
    parent      = prob,
}

\newglossaryentry{expfamsuffstat}{
    name        = \ensuremath{T(\gls{obs})} ,
    description = {sufficient statistics of exponential family} ,
    type        = abbrev,
    parent      = prob,
}

\newglossaryentry{expfamlogpartition}{
    name        = \ensuremath{A} ,
    description = {log parition function of exponential family (depends on \gls{expfamnatparam})} ,
    type        = abbrev,
    parent      = prob,
}

\newglossaryentry{incov}{
    name        = \ensuremath{\gls{cov}_{\mathrm{i}}},
    description = {input covariance matrix} ,
    type        = abbrev,
    parent      = prob,
}

\newglossaryentry{inoutcov}{
    name        = \ensuremath{\gls{cov}_{\mathrm{io}}},
    description = {input-output covariance matrix} ,
    type        = abbrev,
    parent      = prob,
}

\newglossaryentry{clloss}{
    name        = \ensuremath{\mathcal{L}_{\mathrm{INCE}}} ,
    description = {InfoNCE loss function} ,
    type        = abbrev,
}

\newglossaryentry{cllossgen}{
    name        = \ensuremath{\mathcal{L}_{\mathrm{AINCE}}} ,
    description = {AnInfoNCE loss function} ,
    type        = abbrev,
}

\newglossaryentry{alignloss}{
    name        = \ensuremath{\mathcal{L}_{\mathrm{align}}} ,
    description = {alignment term in \gls{clloss}} ,
    type        = abbrev,
}

\newglossaryentry{uniformloss}{
    name        = \ensuremath{\mathcal{L}_{\mathrm{uniform}}} ,
    description = {uniformity term in \gls{clloss}} ,
    type        = abbrev,
}

\newglossaryentry{ball}{
    name        = \ensuremath{\mathcal{B}} ,
    description = {ball} ,
    type        = abbrev,
}

\newglossaryentry{h}{
    name        = \ensuremath{\mathrm{\boldsymbol{h}}} ,
    description = {composition of encoder and decoder, \ie, $\gls{enc}\circ\gls{dec}$} ,
    type        = abbrev,
}

\newglossaryentry{l2latent}{
    name        = \ensuremath{\hat{\gls{latent}}} ,
    description = {$\ell_2$-normalized latent} ,
    type        = abbrev,
    parent     = latents,
}

\newglossaryentry{temp}{
    name        = \ensuremath{{\boldsymbol{\tau}}} ,
    description = {temperature in \gls{clloss}} ,
    type        = abbrev,
}

\makeglossaries

%% file: acronyms.tex
\newacronym{mpa}{MPA}{Measure Preserving Automorphism}
\newacronym{iid}{i.i.d.}{independent and identically distributed}
\newacronym{vmf}{vMF}{von Mises-Fisher}
\newacronym{nivmf}{nivMF}{non-isotropic von Mises-Fisher}
\newacronym{pd}{PD}{positive definite}
\newacronym{psd}{PSD}{positive semi-definite}
\newacronym{nd}{ND}{negative definite}
\newacronym{nsd}{NSD}{negative semi-definite}

\newacronym{ode}{ODE}{Ordinary Differential Equation}
\newacronym{pde}{PDE}{Partial Differential Equation}

\newacronym{lhs}{LHS}{left hand side}
\newacronym{rhs}{RHS}{right hand side}

\newacronym{rv}{RV}{random variable}

\newacronym{ae}{AE}{AutoEncoder}
\newacronym{lae}{LAE}{Linear Autoencoder}
\newacronym{vae}{VAE}{Variational Autoencoder}
\newacronym{cvvae}{CV-VAE}{Constant-Variance Variational Autoencoder}
\newacronym{ivae}{iVAE}{Identifiable Variational Autoencoder}
\newacronym{rae}{RAE}{Regularized Autoencoder}
\newacronym{grae}{GRAE}{Gaussian Regularized Autoencoder}
\newacronym{lvm}{LVM}{Latent Variable Model}

\newacronym[longplural=Gaussian Processes]{gp}{GP}{Gaussian Process}
\newacronym{gplvm}{GPLVM}{Gaussian Process Latent Variable Model}
\newacronym{rbf}{RBF}{Radial Basis Function}

\newacronym{kld}{KL}{Kullback-Leibler Divergence}
\newacronym{elbo}{{\text{\upshape ELBO}}}{evidence lower bound}
\newacronym{pca}{PCA}{Principal Component Analysis}
\newacronym{ppca}{PPCA}{Probabilistic Principal Component Analysis}
\newacronym{ebm}{EBM}{Energy-Based Model}
\newacronym{cca}{CCA}{Canonical Correlation Analysis}

\newacronym{mi}{MI}{Mutual Information}

\newacronym{icm}{ICM}{Independent Causal Mechanisms}
\newacronym{sms}{SMS}{Sparse Mechanism Shift}
\newacronym{sem}{SEM}{Structural Equation Model}
\newacronym{lingam}{LiNGAM}{Linear Non-Gaussian Acyclic Model}
\newacronym{dag}{DAG}{Directed Acyclic Graph}
\newacronym{anm}{ANM}{Additive Noise Model}
\newacronym{cd}{CD}{Causal Discovery}
\newacronym{crl}{CRL}{Causal Representation Learning}

\newacronym{ica}{ICA}{Independent Component Analysis}
\newacronym{nlica}{NLICA}{nonlinear Independent Component Analysis}
\newacronym{bss}{BSS}{Blind Source Separation}

\newacronym{ima}{{\text{\upshape IMA}}}{Independent Mechanism Analysis}
\newacronym{igci}{IGCI}{Information Geometric Causal Inference}
\newacronym{cdf}{CdF}{Causal de Finetti}

\newacronym{nce}{NCE}{Noise Contrastive Estimation}
\newacronym{pcl}{PCL}{Permutation-Contrastive Learning}
\newacronym{tcl}{TCL}{Time-Contrastive Learning}
\newacronym{iia}{IIA}{Independent Innovation Analysis}

\newacronym{ar}{AR}{AutoRegressive}
\newacronym{var}{VAR}{Vector AutoRegressive}
\newacronym{nvar}{NVAR}{Nonlinear Vector AutoRegressive}

\newacronym{ai}{AI}{Artificial Intelligence}
\newacronym{ml}{ML}{Machine Learning}
\newacronym{dl}{DL}{Deep Learning}
\newacronym{rl}{RL}{Reinforcement Learning}
\newacronym{ssl}{SSL}{Self-Supervised Learning}

\newacronym{cl}{CL}{Contrastive Learning}
\newacronym{dcl}{DCL}{Debiased Contrastive Learning}
\newacronym{scl}{SCL}{Spectral Contrastive Learning}
\newacronym{gcl}{GCL}{Graph Contrastive Learning}
\newacronym{alphacl}{$\alpha$-CL}{$\alpha$-Contrastive Learning}
\newacronym{arcl}{ArCL}{Augmentation-robust Contrastive Learning}
\newacronym{hn}{HN}{hard negative}
\newacronym{aninfonce}{AnInfoNCE}{Anistropic InfoNCE}

\newacronym{vince}{VINCE}{Variational InfoNCE}
\newacronym{rince}{RINCE}{Robust InfoNCE}
\newacronym{aggnce}{AggNCE}{Aggregated InfoNCE}
\newacronym{mcinfonce}{MCInfoNCE}{Monte-Carlo InfoNCE}

\newacronym{gmc}{GMC}{Geometric Multimodal Contrastive Learning}
\newacronym{looc}{LooC}{Leave-one-out Contrastive Learning}
\newacronym{npc}{NPC}{Negative-Positive Coupling}
\newacronym{cpc}{CPC}{Contrastive Predictive Coding}
\newacronym{nlp}{NLP}{Natural Language Processing}
\newacronym{gdl}{GDL}{Geometric Deep Learning}
\newacronym{msn}{MSN}{Masked Siamese Networks}
\newacronym{ifm}{IFM}{Implicit Feature Modification}

\newacronym{dnn}{DNN}{Deep Neural Network}
\newacronym{nn}{NN}{Neural Network}
\newacronym{ann}{ANN}{Artificial Neural Network}

\newacronym{nc}{NC}{Neural Collapse}
\newacronym{cdt}{CDT}{Class-Dependent Temperature}
\newacronym{mlp}{MLP}{Multi-Layer Perceptron}
\newacronym{fc}{FC}{Fully Connected}

\newacronym{cn}{conv}{Convolutional layer}
\newacronym{cnn}{CNN}{Convolutional Neural Network}
\newacronym{gnn}{GNN}{Graph Neural Network}

\newacronym{rnn}{RNN}{Recurrent Neural Network}
\newacronym{lstm}{LSTM}{Long Short-Term Memory}
\newacronym{gru}{GRU}{Gated Recurrent Unit}
\newacronym{relu}{ReLU}{Rectified Linear Unit}
\newacronym{bn}{BN}{Batch Normalization}
\newacronym{dbn}{DBN}{Decorrelated Batch Normalization}

\newacronym{gan}{GAN}{Generative Adversarial Network}

\newacronym{sgd}{SGD}{Stochastic Gradient Descent}
\newacronym{adam}{ADAM}{Adaptive Moment Estimation}
\newacronym{svd}{SVD}{Singular Value Decomposition}
\newacronym{wls}{WLS}{Weighted Least Squares}

\newacronym{sam}{SAM}{Sharpness-Aware Minimization}
\newacronym{samba}{SAMBA}{SAM-Based Autoencoder}
\newacronym{vi}{VI}{Variational Inference}
\newacronym{mfvi}{MFVI}{Mean Field Variational Inference}

\newacronym{dgp}{DGP}{Data Generating Process}
\newacronym{map}{MAP}{Maximum A Posteriori}
\newacronym{mle}{MLE}{Maximum Likelihood Estimation}

\newacronym{etf}{ETF}{Equiangular Tight Frame}

\newacronym{mse}{MSE}{Mean Squared Error}
\newacronym{mae}{MAE}{Mean Absolute Error}
\newacronym{ce}{{\text{\upshape CE}}}{cross-entropy}

\newacronym{sid}{SID}{Structural Intervention Distance}
\newacronym{shd}{SHD}{Structural Hamming Distance}

\newacronym{mcc}{MCC}{Mean Correlation Coefficient}
\newacronym{mig}{MIG}{Mutual Information Gap}
\newacronym{dci}{DCI}{Disentanglement Completeness Informativeness score}

\newacronym{arc}{ARC}{Average Relative Confusion}
\newacronym{acr}{ACR}{Average Confusion Ratio}

\newacronym{api}{API}{Application Programming Interface}
\newacronym{cpu}{CPU}{Central Processing Unit}
\newacronym{gpu}{GPU}{Graphics Processing Unit}

\newacronym{gt}{{\text{\upshape GT}}}{ground truth}
\newacronym{lti}{LTI}{Linear Time-Invariant}

\newacronym{winfonce}{W-InfoNCE}{Weighted-InfoNCE}

%% file: abstract.tex

Prior theory work on Contrastive Learning via the InfoNCE loss showed that, under certain assumptions, the learned representations recover the ground-truth latent factors. We argue that these theories overlook crucial aspects of how CL is deployed in practice. Specifically, they either assume equal variance across all latents or that certain latents are kept invariant. However, in practice, positive pairs are often generated using augmentations such as strong cropping to just a few pixels. Hence, a more realistic assumption is that all latent factors change with a continuum of variability across all factors. We introduce AnInfoNCE, a generalization of InfoNCE that can provably uncover the latent factors in this anisotropic setting, broadly generalizing previous identifiability results in CL. We validate our identifiability results in controlled experiments and show that AnInfoNCE increases the recovery of previously collapsed information in CIFAR10 and ImageNet, albeit at the cost of downstream accuracy. 
Finally, we discuss the remaining mismatches between theoretical assumptions and practical implementations.

%% file: main_text.tex
\section{INTRODUCTION} \label{sec:introduction}
    \emph{\gls{ssl}} employs surrogate objectives to learn useful representations from large, unlabeled datasets. A particularly successful and simple sub-family, called \emph{\acrfull{cl}}, minimizes the representational distance between similar samples (\eg, augmentations of the same sample) while maximizing the distance between dissimilar samples (drawn i.i.d.)~\citep{oord_representation_2019,chen_simple_2020,chen2020improved,caron2020unsupervised,chen2020big}. Two similar samples form a \emph{positive pair}, while two dissimilar samples form a \emph{negative pair}. Most of today's successful \gls{ssl} techniques build on this principle, including non-contrastive ones that only use positive pairs~\citep{chen_exploring_2020,grill_bootstrap_2020,zbontar_barlow_2021}.

    \emph{InfoNCE}~\citep{oord_representation_2019} is a commonly used objective for CL.
    \citet{zimmermann_contrastive_2021} have shown that under certain assumptions put on the \emph{\gls{dgp}}, a model can recover the ground-truth latent factors if trained with InfoNCE.
    Their identifiability result requires an isotropic change across all latent factors in the positive pair.
    \Citet{von_kugelgen_self-supervised_2021} relaxed the isotropic assumption on the positive conditional distributions 
    to two partitions: content ($\delta$-distribution) and style (non-degenerate distribution), yielding a 
    weaker block-identifiability result.

\captionsetup{aboveskip=18pt,belowskip=-8pt}
\begin{figure*}[t]
    \centering
    \includegraphics[width=\linewidth]{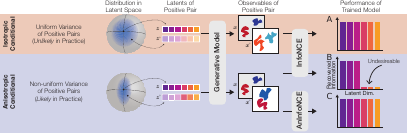}
    \caption{\textbf{Illustration of the Mismatch Between the Standard \gls{cl} Model and Practice.} \textbf{A:} \gls{cl} with the commonly used InfoNCE objective is identifiable when all latents change to the same extent across the positive pair~\citep{zimmermann_contrastive_2021}, which is unlikely to happen in practice. \textbf{B:} The more likely scenario, when augmentations affect different latents to a different extent, leads to dimensional collapse and information loss. \textbf{C:} Our proposed objective, \ourloss, models features that can vary to a different degree in the positive pair, avoiding collapse.}
    \label{fig:title}
\end{figure*}

    In practice, positive pairs for \gls{cl} are generated using augmentations. For instance, strong cropping, which often extracts just a tiny area of the image, is argued to be crucial for high downstream performance \citep{chen_simple_2020}. Assuming an underlying latent variable model, these augmentations correspond to changes in latent factors, modeled by the positive conditional distribution. Cropping discards a significant amount of information and, hence, harshly changes some latent factors. 
    Another augmentation, Gaussian blurring, causes moderate perturbations to the sharpness of objects.
    Lastly, latents alluding to class information are invariably left intact.
    We posit that augmentations imply anisotropic changes to different latent factors, a setting previously not covered theoretically.    
    In fact, we believe that this gap between theory and practice can in part explain why empirical work shows that \gls{cl} loses information~\citep{chen_simple_2020, jing_understanding_2022}, \emph{despite} the identifiability guarantees of~\citet{zimmermann_contrastive_2021,von_kugelgen_self-supervised_2021}.    

    The information loss in CL hinders the development of task-agnostic models, which require an accurate and holistic low- and high-level representation of their inputs.
    The failure of InfoNCE to capture style latents is already apparent in popular vision-language models (VLMs) trained by adaptations of InfoNCE, such as CLIP \citep{radford2021learning}.
    By focusing on a few semantic concepts and dismissing low-level style information, CLIP-based VLMs struggle with low-level vision \citep{rahmanzadehgervi2024vision}, scene compositionality \citep{yuksekgonul2023and}, and attribute binding tasks \citep{trusca2024object}.
    Identifiability considers how to reconstruct \emph{all}  latent factors underlying the inputs, aiming to capture general-purpose representations, suitable for a plethora of downstream tasks.

    Our work serves two purposes. Firstly, we push the frontier of identifiability research by generalizing existing results to anisotropic latent distributions, which allow for more expressive generative processes. Equally importantly, we design a non-trivial extension of InfoNCE, coined \ourloss, that is provably able to capture these inhomogeneous latents. We validate our results in controlled experiments and show that \ourloss successfully increases latent recovery on CIFAR10 and ImageNet, although at the expense of downstream classification accuracy. We perform ablations to understand the remaining gap caused by the accuracy--identifiability trade-off. 
    Our \textbf{contributions}:
    %
    %
   %
    \begin{itemize}[nolistsep,leftmargin=*]
        \item We introduce \ourloss, a generalized contrastive loss, that provably recovers latents with augmentation-induced anisotropic variances (\cref{subsec:main_result}).
       \item We extend the identifiability result to hard negative mining (\cref{seq:hn}) and loss ensembles (\cref{sec:ensemble}).
        \item We verify our loss in synthetic and well-controlled image experiments, having full knowledge of the ground-truth generative process (\cref{subsec:lvm} and \cref{sec:vae_mnist}). 
        \item We further demonstrate the efficacy of \ourloss on CIFAR10 and ImageNet in recovering the latent factors, but observe a trade-off with downstream classification accuracy, done with linear probing (\cref{subsec:real_world_exp}).
        \item We link the accuracy--identifiability trade-off on real-world data to augmentations and analyze remaining mismatches between \gls{cl} theory and practice (\cref{sec:analysis}).
    \end{itemize}

\section{BACKGROUND} \label{sec:bg}

\paragraph{Contrastive Learning.}
    \acrfull{cl} is an \gls{ssl} paradigm where an encoder $\gls{enc}: \gls{Obs}\to \gls{Latent}$ learns to map observations \gls{obs} to latent vectors \gls{latent}. \gls{cl} uses positive pairs $(\gls{obs}, \gls{obspos})$ (similar data points, \eg, augmentations of the same sample), and negative pairs $(\gls{obs}, \gls{obsneg})$ (dissimilar data points, \eg, uniformly drawn from the dataset).
    \gls{cl} mostly employs the InfoNCE loss~\citep{sohn2016improved, gutmann_noise-contrastive_nodate,oord_representation_2019}: 
    \begin{align}
        &\gls{clloss}(\gls{enc}) = \label{eq:cl}\\
        &=\!\!\underset{{\substack{\gls{obs}, \gls{obspos} \\ \{\obsnegi[i]\}}}}{\expectation{}}\!
        \brackets{-\ln\!\dfrac{\expinnerpos}{\expinnerpos \!+\! \sum_{i=1}^{\gls{numneg}}\expinnerneg}}\!\!, \nonumber
    \end{align}
    where \gls{temp} is a scalar temperature, and $\gls{obspos}\sim p(\gls{obspos}|\gls{obs})$ and $(\gls{obs}, \gls{obsneg})\sim^{\acrshort{iid}}p(\gls{obs})$ are the positive and negative pairs observed.

    In order to analyze the behavior of \gls{cl}, we surmise that the observations come from a Data Generating Process (\gls{dgp}) $\gls{dec}: \gls{Latent}\to \gls{Obs}$.
    The samples follow a marginal distribution $p\parenthesis{\gls{latent}}$ over the latent space \gls{Latent}, and the positive samples follow a conditional \clconditional. Under certain assumptions, \gls{cl} inverts the \gls{dgp}, \ie, the composition ${\gls{h} = \gls{enc}\circ\gls{dec}}$ is a trivial map (usually an affine linear map, depending on the assumptions)~\citep{zimmermann_contrastive_2021}. The identifiability of \gls{ssl} is further analyzed by \citet{von_kugelgen_optimal_2019,lyu_latent_2021} and \citet{cui_aggnce_nodate}, with insightful connections between contrastive and non-contrastive methods in~\citet{balestriero_contrastive_2022,garrido_duality_2022,assran_hidden_2022}.

\paragraph{Content-style Partitioning of \gls{Latent}.} 
    \Citet{von_kugelgen_self-supervised_2021} present an identifiability result for a partitioning of $\gls{latent}\in\gls{Latent}$ into invariant (\emph{content}) $\gls{content}\in\gls{Content}$ and variant (\emph{style}) ${\gls{style}\in\gls{Style}}$ dimensions with $\gls{Latent}=\gls{Content}\times\gls{Style}$ 
    (see~\cref{sec:app_generalized_content}). The content latents \gls{content} in the positive pairs are assumed to follow a $\delta$-distribution, whereas the distribution of \gls{style} is not degenerate, yielding  $\clconditional=\delta\parenthesis{{\gls{latent}^{+}_c}-\gls{content}}p({\gls{latent}^{+}_s}|\gls{style})$. 
    Note that \citet{jing_understanding_2022} implicitly use a similar distinction---defined as different augmentation amplitudes. Thus, their approach is akin to generalized content and style variables with variances $\vars\gg\varc>0$; \citet{dubois_improving_2022} also have a similar approach. Other means to partition the latent variables are tied to the ``simplicity" of features~\citep{chen_intriguing_2021} or sparsity arguments~\cite{wen_toward_2021}. In concurrent work, \citet{eastwood_self-supervised_2023} propose a \textit{relative} notion of content and style, assuming the presence of multiple environments, each with distinct 
    content and style dimensions.

    
\paragraph{Practical Shortcomings of \acrshort{cl}.}
    Despite the identifiability guarantees of \gls{cl} methods~\citep{zimmermann_contrastive_2021,von_kugelgen_self-supervised_2021,cui_aggnce_nodate}, practitioners struggle to make models learn all underlying latent factors. \citet{chen_simple_2020} noted that the representations could become invariant to some input transformations. Several works connected data augmentations to failing to learn all latent factors~\citep{jing_understanding_2022,wang_chaos_2022,bendidi_no_2023,wu_mutual_2020,haochen_provable_2021,cosentino_toward_2022,wagner_importance_2022,zhai_understanding_2023}.
    
    A potential way to improve augmentation readout (a proxy for capturing the information in highly-varying, \ie, style, latents) is by introducing inductive biases. \citet{lee_improving_2021} add a regularizer to predict the augmentation parameters between two views, whereas \citet{dangovski_equivariant_2022} use a regularizer to induce equivariance by predicting rotation parameters. \citet{xiao_what_2021} introduce \gls{looc} with multiple latent spaces. They train multiple embeddings by leaving out one augmentation at a time, thus inducing different invariances, which results in better downstream and augmentation readout performance.
    \citet{zhang_rethinking_2022} propose a similar strategy with different augmentation partitions and calculate the contrastive loss at different hierarchical levels in the encoder, meant to induce invariances at different feature levels. Recently, \citet{eastwood_self-supervised_2023} proposed a similar method to \gls{looc}~\citep{xiao_what_2021}, but with strong theoretical guarantees.

\section{IDENTIFIABLE ANISOTROPIC CL} \label{sec:ident_theory}

\begin{scpcmd}[
\newscpcommand{\Z}{}{\gls{Latent}}
\newscpcommand{\z}{}{\gls{latent}}
\newscpcommand{\pz}{}{\gls{latentpos}}
\newscpcommand{\nz}{[1][i]}{\latentnegi[#1]}
\newscpcommand{\Lam}{}{\gls{diagtemp}}
\newscpcommand{\x}{}{\gls{obs}}
\newscpcommand{\px}{}{\gls{obspos}}
\newscpcommand{\nx}{[1][i]}{\obsnegi[#1]}
\newscpcommand{\X}{}{\gls{Obs}}
\newscpcommand{\M}{}{\gls{numneg}}
\newscpcommand{\g}{}{\gls{dec}}
\newscpcommand{\f}{}{\gls{enc}}
\newscpcommand{\D}{}{\gls{obsdim}}
\newscpcommand{\hz}{}{\gls{latentrec}}
]

\subsection{Setup and Intuition}\label{subsec:theory_intuition}
    We argue that the augmentations usually employed in applied Contrastive Learning (\gls{cl}) do not change all latents evenly. For example, color distortions do not drastically affect latents encoding semantic information or edges. We use a \gls{lvm} to model \gls{cl}, where a positive conditional $\dist(\pz|\z)$ relates the positive pairs \pz{} to the anchor sample \z. The augmentations used in practice induce a specific form of the positive conditional $\dist(\pz|\z)$, \eg, by assigning larger variance to latents strongly affected by augmentations. As the current \gls{cl} theory does not account for augmentations that affect latents to different levels, we propose an anisotropic model for \gls{cl}.  

    Following the practice of CL, we assume that the latent space is a $(\gls{latentdim}-1)$-dimensional hypersphere:  $\Z = \Sd[\gls{latentdim}-1]$. As proposed by \citet{zimmermann_contrastive_2021}, the anchor latent $\z$, and all negative pairs $\{\nz[i]\}$ are assumed to be uniformly distributed on $\Z$. Akin to \citet{kirchhof_non-isotropic_2022}, 
    we allow
    non-uniform variances for the individual latent factors in the positive pair $\pz$.
    We model this in the positive conditional by scaling the latent factors with \emph{unknown} positive \emph{concentration parameters}, collected in the diagonal matrix $\Lam$:
    \vspace{-4pt}
    \begin{equation}\label{eq:weighted_cl_cond}
        \dist(\z) = \mathrm{const},\quad \dist(\pz|\z) \propto e^{-(\pz-\z)^\top \Lam (\pz -\z)},
    \end{equation}
    \vspace{-1pt}
    where a higher value of $\Lam_{ii}$ means a higher concentration of $\pz$ around the anchor $\z$ along \makebox{dimension $i$}.
    Introducing \gls{diagtemp} in the conditional generalizes previous works: $\Lam = \Id{\gls{latentdim}}$ recovers the \gls{vmf} conditional from~\citet{zimmermann_contrastive_2021}, whereas two latent partitions, namely content (high $\Lam_{ii}$) and style (low $\Lam_{jj}$), yields the model of \citet{von_kugelgen_self-supervised_2021}.
    Our new conditional incorporates an entire spectrum of distributions between the two edge cases.

    We assume that observations $\x,\px,\{\nx[i]\}$ are the result of passing the respective latents through an invertible generator \makebox{$\g:\Z\rightarrow\X\subseteq\rr[\D]$}.
    Our goal is to train an encoder $\f:\rr[\D]\rightarrow\Z$ such that the reconstructed latents, $\hz = \f(\x)$, recover $\z$ up to permissible transformations. To train \gls{enc}, we introduce our new loss: 
    \vspace{-3pt}
    \begin{align}
        &\gls{cllossgen}(\f,\gls{diagtempinfer}) = \label{eq:weightedcl}\\
        & = \!\!\underset{\substack{\x, \px\\ \{\nx[i]\}}}{\expectation{}}\!
         \left[\!-\!\ln\!\dfrac{e^{-\normsquared{\f(\px) - \f(\x)}_{\gls{diagtempinfer}}}}{e^{{-\normsquared{f(\px) - \f(\x)}_{\gls{diagtempinfer}}}} \!+\! \sum_{i=1}^{\M} e^{-\normsquared{\f(\nx[i]) - \f(\x)}_{\gls{diagtempinfer}}}}\right]\!\!, \nonumber
    \end{align}
    \vspace{-0pt}
    where $\gls{diagtempinfer}$ is a \emph{trainable} diagonal scaling matrix, aimed to capture $\Lam$, and the weighted and squared $\ell_2$-distance, $\!{}-\!\normsquared{\f(\px\!) \!- \f(\x)}_{\gls{diagtempinfer}} \!=\! -(\f(\px\!) -\! \f(\x))\!^\top \gls{diagtempinfer} (\f(\px\!) - \f(\x))$, is the \emph{similarity function}.
    The idea behind \ourloss is that the similarity function should be flexible enough to accommodate latents anisotropically perturbed by the augmentations. We defer the details to \cref{sec:identifiability_proofs} and summarize our assumptions here: 

    \begin{restatable}[Anisotropic \gls{cl} on $\mathbb{S}^{\gls{latentdim}-1}$]{assum}{assumnonisotropiccl}\label{assum:non_isotropic_cl}
        The \textbf{\gls{dgp}} satisfies (\cf \cref{def:adgp} for details):
        \begin{enumerate}[leftmargin=*]
            \item \label{assumptions:first}$\gls{Latent}\!\!=\!\!\Sd[\gls{latentdim}-1]\!$, $\gls{Obs} \!\!\subseteq\!\!  \rr[\gls{obsdim}]$ are latent and observation spaces.
            \item Anchors are uniform on $\Z$, \ie, $ 
            \clmarginal = \mathrm{const}$.
            \item  For $\Lam\in\rr[\gls{latentdim}\times\gls{latentdim}]$ \acrshort{pd} diagonal matrix, the positive conditional \clconditional has the form of \eqref{eq:weighted_cl_cond}.
            \item \label{item:negative_cond}Negative pairs $\{\nz[i]\}$ are uniform on $\Z$.
            \item A continuous, invertible generator function $\gls{dec}$ maps $  \z, \z^{\pm} \!\in\! \Z \overset{\g}{\mapsto}\x, \x^{\pm} \!\in\! \X$.
        \end{enumerate}
        The \textbf{model} satisfies (\cf \cref{def:aclp} for details):
        \begin{enumerate}[resume,leftmargin=*]
            \item The encoder $\gls{enc} : \rr[\D] \to \Sd[\gls{latentdim}-1]$ is continuous.
            \item $\gls{diagtempinfer} \in \rr[\gls{latentdim}\times\gls{latentdim}]$ is a \acrshort{pd} diagonal matrix. 
            \item \label{assumptions:last}$\gls{enc}$ and $\gls{diagtempinfer}$ are trained with $\gls{cllossgen}$, \ie \eqref{eq:weightedcl}.
        \end{enumerate}
        Collectively \ref{assumptions:first}-\ref{assumptions:last}  define an \emph{anisotropic \gls{cl} problem}.
    \end{restatable}

\subsection{Identifiability of Anisotropic \gls{cl}}\label{subsec:main_result}

    When studying identifiability, we aim to recover the ground-truth latent $\z$ from observations $\x=\g(\z)$ with as little ambiguity as possible. 
    Our main contribution is proving the identifiability of the \gls{dgp} from \cref{assum:non_isotropic_cl}, generalizing \citet{zimmermann_contrastive_2021,von_kugelgen_self-supervised_2021}:
    \begin{theorem}[Identifiability of Anisotropic \gls{cl}]\label{prop:min_ce_maintains_weighted_mse}
        Under \cref{assum:non_isotropic_cl}, if a pair $(\f,\gls{diagtempinfer})$ minimizes \gls{cllossgen}~\eqref{eq:weightedcl}, then $\f\circ\g$ is a block-orthogonal transformation, where each block acts on latents with equal weight $\gls{diagtemp}_{ii}$.
        In other words, \gls{latent} is identified up to a block-orthogonal transformation.
    \end{theorem}
    
    \cref{prop:min_ce_maintains_weighted_mse} means that if the encoder $\f$, 
    usually parametrized by a neural network, is expressive enough to minimize \ourloss, then the reconstructed latent $\hz = \f(\x)$ is related to the ground-truth $\z$ via a simple orthogonal linear transformation. This simple ambiguity remains inconsequential as solving any downstream task involves at least one linear layer trained on top of the backbone $\f$. 

    \vspace{-5pt}
    \begin{proof}[Proof Sketch (full proof in \cref{subsec:ident_base})]
        We rewrite \ourloss into a categorical cross-entropy of
        predicting the correct index of the positive pair from randomly shuffled data $\px,\{\nx[i]\}$. The minimum loss is attained iff the similarity function matches the log-density-ratio between the positive and negative conditionals (\cref{theo:bayes_optima_cl}). This provides us with a functional equation that we solve for $(\f,\gls{diagtempinfer})$.
    \end{proof}

\end{scpcmd}

\subsection{Extending Identifiability to Hard Negative Mining} \label{seq:hn}
    Our generalized result in \cref{subsec:main_result} accounts for anisotropically varying latent factors. 
    Here, we build on \cref{prop:min_ce_maintains_weighted_mse} to theoretically model the practice of \gls{hn} mining.
    
    \gls{hn} mining effectively uses a negative conditional distribution to select negative samples more similar to the anchor, instead of sampling from a uniform marginal. Put differently, \gls{hn} mining makes the classification problem of positive/negative samples \textit{more difficult}, which is reasoned to help avoid latent collapse~\citep{wu_mutual_2020,chuang_debiased_2020,robinson_contrastive_2021,zheng_contrastive_2021,wang_understanding_2021,robinson_can_2021,khosla_supervised_2020, yuksekgonul2023when}.
    Our proposed anisotropic loss still identifies the ground-truth latents, when the negative conditional has the same form as \eqref{eq:weighted_cl_cond} (we denote its concentration parameter as $\gls{diagtemp}^-$, and the one of the positive conditional as $\gls{diagtemp}^+$). Thus, we further generalize the setting of  \cref{subsec:main_result}, which is subsumed as $\gls{diagtemp}^-=\boldsymbol{\mathrm{0}}$ (\ie, a uniform marginal). Intuitively, with a negative conditional, the density ratios will be the same as with a uniform marginal and a positive conditional with $\gls{diagtemp}=\gls{diagtemp}^+-\gls{diagtemp}^-$. 
    \begin{corollary}[Identifiability with HN mining. Proof in \cref{subsec:theory_hard_negatives}]\label{cor:hn_ident}
        Under \cref{assum:non_isotropic_cl}, but with a negative conditional in \makebox{\ref{item:negative_cond}} as \eqref{eq:weighted_cl_cond} with a concentration parameter $\gls{diagtemp}^{-}$, such that $\gls{diagtemp}=\gls{diagtemp}^+-\gls{diagtemp}^-$ has full rank, the pair $(\gls{enc},\gls{diagtempinfer})$  minimizing \gls{cllossgen} identifies the latent factors up to a block-orthogonal transformation.
    \end{corollary}

We include experimental results for hard negative mining in \cref{app:mitigation}. We empirically show that our theory correctly predicts the loss optimum at $\gls{diagtemp} = \gls{diagtemp}^+-\gls{diagtemp}^-$, and that fine-tuning an encoder on a concatenation of hard negative samples and regular negative samples improves the identifiability score.


\subsection{Extending Identifiability to Ensemble \ourloss} \label{sec:ensemble}
    Recent works started exploring both theoretically~\citep{eastwood_self-supervised_2023,kirchhof_probabilistic_2023} and empirically~\citep{xiao_what_2021,zhang_rethinking_2022} the advantages of ensemble \gls{cl} losses, where each loss component is calculated with different (potentially partially overlapping) data augmentations. We show that an ensemble version of \ourloss is also identifiable.
    For more details, see \cref{app:subsec_theory_ensemble}.

    \begin{scpcmd}[
        \newscpcommand{\Lam}{}{\gls{diagtemp}}
        \newscpcommand{\infLam}{}{\gls{diagtempinfer}}
        \newscpcommand{\ensLam}{[1]}{\Lam^{\raisemath{-0.1ex}{#1}}}
        \newscpcommand{\ensinfLam}{[1]}{\infLam^{\raisemath{-0.1ex}{#1}}}
        \newscpcommand{\f}{}{\gls{enc}}
        ]
    \begin{corollary}[Identifia.~of ensemble \ourloss]\label{cor:ensemble_ident}
        Assume $k$ \glspl{dgp}, such that each satisfy \cref{assum:non_isotropic_cl}. Assume that they share \gls{dec}, but have different concentration parameters $\{\ensLam{i}\}$, for $1\leq i \leq k$. Assume a shared encoder \gls{enc} and $k$ learnable $\ensinfLam{i}$ parameters. Then, the tuple $(\gls{enc},\{\ensinfLam{i} \})$ minimizing the ensemble loss $\sum_i^k\gls{cllossgen}(\f, \ensinfLam{i})$ identifies the latents up to a block-orthogonal transformation.
    \end{corollary}
    %
    \end{scpcmd}
        
    \Cref{cor:ensemble_ident} shows that using an ensemble with proper augmentations could lead to stronger identifiability guarantees (\eg, up to permutation, \cf \cref{remark:ensemble_ident_permutations}).
    Ideally, we could design such sets of augmentations, but it is not evident that this is possible in practice.

We include experimental results for ensemble AnInfoNCE in~\cref{app:mitigation}.
To verify~\cref{cor:ensemble_ident}, we define two DGPs that share a uniform anchor distribution and differ in their respective conditional distributions. We sample data from both DGPs, calculate the individual losses, and train the encoder on the sum of both losses. We see improved linear identifiability scores with ensemble \ourloss, compared to training the encoder with InfoNCE or AnInfoNCE.

\section{EXPERIMENTS}\label{sec:experiments}

    We now validate our theoretical results in experiments with increasing complexity. This section focuses on our main result and demonstrates the efficacy of AnInfoNCE in various settings.
    We include additional experimental ablations, results for hard negative mining and ensembling AnInfoNCE and detailed compute requirements for all our experiments in the Appendix. 
    
\subsection{Synthetic Experiments}
\label{subsec:lvm}
\captionsetup{aboveskip=15pt,belowskip=-0pt}
    \begin{figure*}[t]
        \centering
        \includegraphics[width=1.0\linewidth]{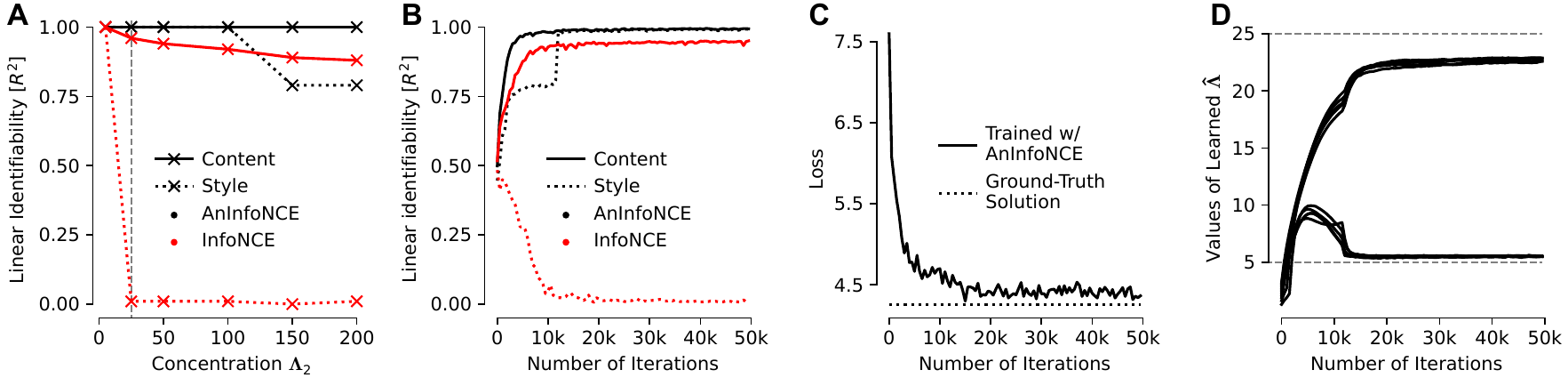}
        \caption{\textbf{Behavior of \ourloss on Synthetic Data.} ($\Lambda_1=5$) \textbf{A:} We maintain high \gls{r2}-scores with \ourloss for both content and style dimensions, while the style dimensions are lost when training with the regular InfoNCE loss. For $\Lambda_2=25$ (dotted black vertical line), we show: \textbf{B:} The evolution of the linear \gls{r2} scores during training; \textbf{C:} \ourloss reaches the global minimum, computed based on ground-truth latents; \textbf{D:} The evolution of the learned $\hat{\gls{diagtemp}}$ values.   \label{fig:toymodel}
        }
    \end{figure*}

\paragraph{Preliminaries.}
    We follow \citet{zimmermann_contrastive_2021} and consider a fully-controlled \acrfull{dgp}, which implements \cref{assum:non_isotropic_cl}. That is, the ground-truth latent space is a hypersphere ($\gls{Latent} = \Sd[\gls{latentdim}-1]$), and positive/negative pairs are defined through conditional sampling in the latent space. We model the generative process \gls{dec} as an invertible \gls{mlp} with the same input, intermediate, and output dimensionality \gls{latentdim} and with Leaky ReLU non-linearities. We set $\gls{latentdim}=10$.
    The encoder \gls{enc} is also parameterized by an \gls{mlp} with the same input, intermediate, and output dimensionality, and with Leaky ReLU non-linearities. 
    The encoder is trained only with observations, which are obtained by passing \gls{latent} through \gls{dec}.
    As we have access to the ground-truth \gls{latent}, we can post-hoc compute the information captured in the inferred latents \gls{latentrec} about the ground-truth \gls{latent}. We do this by fitting a linear map $\mat{A}$ between \gls{latentrec} and \gls{latent} and compute the average coefficient of determination \gls{r2}~\cite{wright1921correlation} between the relevant features (e.g. all of them, content-only dimensions, style-only dimensions) of \mat{A}\gls{latentrec} and \gls{latent}.
    Unless noted otherwise, we use a uniform marginal and a projected Gaussian as the conditional distribution, modeling \cref{eq:weighted_cl_cond} and the assumptions in~\cref{subsec:theory_intuition}.
    

\paragraph{Results.}
    We study the effect of anisotropy in the positive conditional by varying the \makebox{ground-truth $\gls{diagtemp}$}: We keep the value of five latent dimensions fixed at ${\Lambda_1 = 5}$, while we vary the value of the other five dimensions (i.e., $\Lambda_2$), resulting in a ground-truth diagonal ${\gls{diagtemp} = \diag{{\Lambda}_1, \ldots, \Lambda_1, \Lambda_2, \ldots \Lambda_2}}$. 
    While \citet{von_kugelgen_self-supervised_2021} coined the terms ``content'' and ``style'' to refer to invariant and variant latents, respectively, we here relax this notion and use these terms comparatively: In an anisotropic setting, where \gls{diagtemp} is composed of different values, higher (lower) values reflect more content-like (style-like) dimensions.
    We train an encoder on the observed data generated by the process described above and summarize the results in~\cref{fig:toymodel}\textbf{A}. We observe high \gls{r2} scores over a wide range of \gls{diagtemp}-values for both content and style latents when training with our \ourloss loss (markers in \cref{fig:toymodel}\textbf{A} indicate converged training runs). In contrast, regular InfoNCE loss cannot identify the style latents ($R^2\approx 0$).
    We also analyze the setting where the conditionals of five latents are distributed according to $\gls{diagtemp}_2 = 25$ in more detail in \cref{fig:toymodel}\textbf{B-D}.
    We observe that content dimensions are learned before style (\cref{fig:toymodel}\textbf{B}). Our loss converges to the global minimum value, which can be computed based on the ground-truth latents (\cref{fig:toymodel}\textbf{C}).
    When learning \gls{diagtemp}, we observe that the learned \gls{diagtempinfer} values almost match the ground-truth values (dashed black lines in \cref{fig:toymodel}\textbf{D}).

\subsection{VAE Experiments on MNIST}
\label{sec:vae_mnist} 
\paragraph{Preliminaries.}
    We now move closer to actual images by replacing the \gls{mlp} in our generative model with a ten-dimensional Variational Autoencoder (VAE; \citealp{kingma2013auto}) trained on MNIST~\citep{deng2012mnist}.
    This setup is beneficial as we have valid images while fully controlling the \gls{dgp}. Thus, we still sample ground-truth latents, which we then pass through the VAE decoder to generate observations to train on.
    Both the encoder and the decoder are fully connected three-layer \glspl{mlp} with Leaky ReLU activations.
    Generated samples for different concentration parameters \gls{diagtemp} are shown in~\cref{app:mnist}, ~\cref{fig:mnist_samples}.
    
\paragraph{Results.}
    We start with the same anisotropic ground-truth $\gls{diagtemp}$ as in \cref{subsec:lvm}.
    When training an encoder on MNIST samples with \ourloss, we observe perfect linear identifiability scores (\gls{r2}) for all latents, whereas training with the standard InfoNCE loss yields zero \gls{r2} scores for the style latents (\cref{fig:mnist_scores_acc}\textbf{A}).
    The identifiability scores correlate with KNN classification accuracy, which is stable across a wide range of \gls{diagtemp}-values when training with \ourloss, but degrades when training with InfoNCE.
    As our theoretical claims are only valid for the representations after the final layer of the model, we evaluate those. This means we deviate from previous literature and do not use a projection head.

    \captionsetup{aboveskip=15pt,belowskip=-1pt}
    \begin{figure*}[t]
        \centering
        \includegraphics[width=\linewidth]{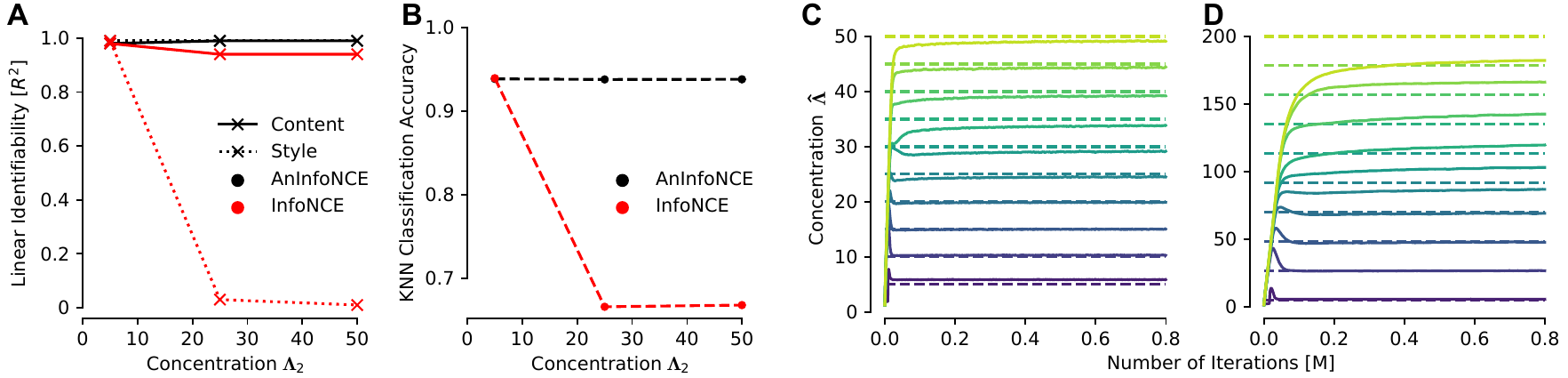}
        \caption{
            \textbf{Behavior of \ourloss on MNIST.}
            \textbf{A:} Linear identifiability (\gls{r2}) scores for an encoder trained on \acrshort{vae}-generated MNIST samples, when varying $\Lambda_2$ and keeping $\Lambda_1 = 5$ for the positive conditional. Style dimensions collapse ($\gls{r2}=0$) for regular InfoNCE. \textbf{B:} KNN-accuracy evaluated on the regular MNIST dataset using the encoders trained as described in \textbf{A}. KNN accuracy degrades when style dimensions are lost.
            \textbf{C \& D:} The evolution of learned $\hat{\gls{diagtemp}}$ during training.
            We set the diagonal entries of $\gls{diagtemp}$ to ten different values by linearly interpolating between $5$ and $50$ (\textbf{C}) or $5$ and $200$ (\textbf{D}). The ground-truth $\gls{diagtemp}$ values are indicated by dashed lines. The learned $\hat{\gls{diagtemp}}$ are shown in the corresponding colors as solid lines.
        }
        \label{fig:mnist_scores_acc}
    \end{figure*}

    Next, we model each latent with a distinct concentration parameter $\gls{diagtemp}_{ii}$ by linearly sampling values between $5$ and a maximum of either $50$ or $200$.
    When training the encoder with \ourloss, we recover all latent dimensions: The linear identifiability scores are $99.7$\% and $99.6$\%, respectively. However, training with the regular InfoNCE loss leads to substantially lower scores of $77.4$\% and $76.3$\%, respectively.
    \cref{fig:mnist_scores_acc}\textbf{C \& D} show how the learned $\hat{\gls{diagtemp}}$ values change during training. We observe that the learned $\hat{\gls{diagtemp}}$ approaches the ground truth, except for very high $\gls{diagtemp}$ values: This is visible in ~\cref{fig:mnist_scores_acc}\textbf{D} when $\gls{diagtemp}_{ii}>100$). We analyze this behavior further in \cref{app:analysis}.

\subsection{Real-world Experiments}\label{subsec:real_world_exp}

    \input{Tables/augmentations_readout}

\paragraph{Preliminaries.}
    On CIFAR10~\citep{krizhevsky2009learning}, we train ResNet18~\citep{he2016resnet} models for 1000 epochs with the code from~\citet{simsiamrepo}, while on ImageNet~\citep{ILSVRC15}, we train ResNet50 models for 100 epochs using the code from~\citet{vicregrepo}.
\paragraph{Results.}
    On real-world datasets, we do not have access to true generative factors and the latent conditional distribution.
    Therefore, we need to create views using regular SimCLR augmentations.
    Another consequence is that we cannot directly measure how well models recover the ground-truth information. Thus, we use proxy evaluations and calculate the linear readout accuracy on the augmentations used during training to judge how well we recover the latent factors.
    Chance-level performance is $50\%$.
    Our experiments on CIFAR10 and ImageNet show that our loss with a trainable $\hat{\gls{diagtemp}}$ leads to much higher readout accuracy on the used augmentations, \ie,  we successfully recover more latent dimensions compared to the regular InfoNCE loss (\cref{tab:augmentations_readout}).
    While we reduce the information loss, this surprisingly does not lead to better downstream accuracy (\cref{tab:augmentations_readout}, \textit{Classes} column).
    Removing $\ell_2$-normalization results in an encoder with the highest augmentations readout accuracy and simultaneously the lowest classification accuracy.
    We analyze this discrepancy, the validity of our assumptions for real-world data, and the required steps for closing the remaining gap between \gls{cl} theory and practice in \cref{sec:analysis}.

\section{ANALYSIS}
\label{sec:analysis}

    Our experiments (\cref{sec:experiments}) show some scenarios in which recovering more (latent) information results in better downstream performance (e.g., readout accuracy on MNIST).
    While \ourloss outperforms InfoNCE in these controlled scenarios, it fails in others: On CIFAR10 and ImageNet, we observe a trade-off between the augmentation readout and linear classification readout accuracies. While higher accuracy in augmentation readout indicates a better capture of style latents (i.e., more latent information is recovered), it does not coincide with higher classification accuracy.
    Scaling up synthetic experiments to ImageNet comes with changes to various training aspects: The encoder's architecture is switched to a ResNet from an invertible MLP, views are sampled using augmentations instead of the ground-truth \gls{dgp}, and the encoder's output dimensionality is orders of magnitude higher.
    To control for these differences, we conduct a minimal-change experiment on MNIST, changing only how the views are generated while the encoder and the training procedure remain the same. We demonstrate that using augmentations instead of generating the views using the ground-truth DGP is the main reason for this trade-off.

    \input{Tables/causal_3dident_normalized}

\paragraph{Minimal-change experiment on MNIST exhibits the same identifiability-accuracy trade-off.}

To visualize how augmentations influence the positive conditional distribution in latent space, we train a VAE on augmented MNIST images and then project these augmented images onto the learned VAE's latent space. 
We observe that the latents of positive pairs do not follow a projected Gaussian distribution and can even follow a bimodal one (see~\cref{fig:analysis_cond} in~\cref{app:analysis}).
Our theory does not cover this case; neither do other theories we are aware of, and so this mismatch presents a fruitful direction for future research on closing the gap between theory and practice.

Further, to better analyze the interplay between augmentations and disentanglement, we train a model either using InfoNCE or \ourloss on MNIST when views are generated using SimCLR augmentations (except color augmentations as MNIST is grayscale):

    \begin{center}
    \small
        \setlength{\tabcolsep}{4pt}
        \begin{tabular}{l c c c c}
        
        Loss &   Class & Lin. id. $[R^2]$   & Crop & Gaussian Blur  \\
             \midrule
 InfoNCE & 	\textbf{0.97}&0.34 &	0.53&	0.58\\
 \ourloss & 	0.94&	\textbf{0.44} &\textbf{0.59}&	\textbf{0.63}\\
        \end{tabular}
    \end{center}

We observe the same accuracy-identifiability trade-off on MNIST as before on ImageNet and CIFAR10: While the augmentations' prediction accuracy improves, the classification accuracy decreases. Using the VAE trained on MNIST, we can also calculate linear identifiability scores, which are higher for \ourloss (but still moderately low). Note that training with \ourloss on MNIST when views are generated by passing ground-truth latents through the VAE leads to perfect linear identifiability scores ($R^2=1$). In that case, we do not observe a trade-off between linear identifiability scores and downstream accuracy.


\paragraph{We observe the accuracy-identifiability trade-off on the Causal 3DIdent dataset.}
    The Causal 3DIdent dataset (C-3DIdent, \citealp{von_kugelgen_self-supervised_2021, zimmermann_contrastive_2021}) is a computer-generated image dataset with seven classes and ten ground-truth latent factors, rendered with Blender \citep{blender}.
    This dataset is more complex than MNIST and thus constitutes an interesting intermediate step between MNIST and visually complex datasets such as ImageNet.

We repeat our analysis from above and train a ResNet18 encoder with InfoNCE and \ourloss on C-3DIdent when views are generated using both the ground truth DGP and data augmentation. 
In addition to the linear identifiability scores, we also report nonlinear identifiability scores calculated using a Kernel Ridge Regression, following~\citet{von_kugelgen_self-supervised_2021}.
   In \cref{tab:3dident}, we observe the same accuracy-identifiability trade-off as on the other datasets when using augmentations: Training with \ourloss leads to better identifiability scores as predicted by our theory. However, training with InfoNCE results in better downstream accuracy. 
    When the ground-truth generative model (DGP) is used for generating the views, \ourloss outperforms InfoNCE in both identifiability and accuracy. Training details can be found in ~\cref{app:cident}.

\paragraph{The accuracy-identifiability trade-off can be mitigated when using only one augmentation.}
To better understand the accuracy-identifiability trade-off, we train models with InfoNCE and \ourloss on Causal 3DIdent using a single augmentation to generate the views. We generate views using small crops (scale parameter between 0.08 and 1.0), large crops (scale parameter between 0.8 and 1.0), by applying color jitter or rotations. We further repeat these experiments using the DGP to generate views by modifying the corresponding latents. We report the results in Table~\ref{tab:3dident_single}.
\input{Tables/causal_3dident_prev}

We observe that modeling singular data augmentations, e.g., only applying crops or color jitter, works very well with AnInfoNCE, and we see no trade-off between classification accuracy and identifiability. For all augmentations we tested, AnInfoNCE is always better at recovering them than InfoNCE.
However, the classification accuracy is generally not very high when using only one augmentation type, except for color jitter where the classification accuracy is 98\% with AnInfoNCE. 
This observation does not hold on real-world data and demonstrates the limitations of the Causal 3D-Ident dataset: When using only color jitter for generating positive samples and training the encoder on CIFAR10, the downstream accuracy drops from 90.9\% to 48.6\%. 
While combining multiple augmentations results in high classification accuracy, it also leads to the accuracy-identifiability trade-off as shown in Table~\ref{tab:3dident}. We hypothesize that modeling two augmentations is more challenging and requires extending our simple diagonal matrix to a more complex function. This likely would make deriving theoretical guarantees more challenging.
When using transformation in the latent codes, AnInfoNCE is always superior to InfoNCE, irrespective of which transformation has been used.

We conclude that AnInfoNCE performs better than InfoNCE on all singular data augmentations we tested and is not sensitive to any of them. However, it struggles when \textit{multiple} augmentations are combined, which is necessary for good downstream accuracy on real-world datasets.

\paragraph{Further ablation studies.}
    We analyze the learning dynamics of \ourloss and conduct a thorough hyperparameter study in~\cref{app:analysis}. We observe that the performance of \ourloss is strongly influenced by sample efficiency, i.e., the batch size and the latent dimensionality.
    Further, we observe that larger concentration parameters also require larger batch sizes. We link this observation to the role of the \textit{augmentation overlap}~\citep{wang_chaos_2022}: Higher $\gls{diagtemp}$ means a more concentrated conditional and requires larger batch sizes for sufficient augmentation overlap. Investigating the evolution of the identifiability scores over time, we observe a step-like behavior in how the latents are learned, corroborating the results by~\citet{simon_stepwise_2023}. We discuss connections between \ourloss and TriCL \citep{zhang2023tricontrastive} and show how the losses compare experimentally in \cref{sec:tricl}.

\paragraph{Discussion.}
    Given our analysis above, we identify the use of augmentations in real-world datasets for generating views as the main cause for the accuracy-identifiability trade-off: Assuming an underlying \gls{lvm}, these augmentations correspond to changes in latent factors, modeled by the positive conditional distribution. This points to a potential cause for unexpectedly low performance: The conditional distribution implied by augmentations may not correspond to the one assumed by the loss. First, the form of the conditional implicitly defined by the augmentations may not be Gaussian or even uni-modal. Second, the corresponding concentration parameters may be high, leading to a flat loss landscape and optimization issues in practice.

\section{CONCLUSION} \label{sec:discussion}
    Our work advances \gls{cl} theory by generalizing InfoNCE. The proposed framework, \ourloss-based \acrfull{cl}, is identifiable and relies on more realistic assumptions on the data, reducing the gap between \gls{cl} theory and practice: It can account for anisotropic latent conditionals that more likely reflect how augmentations affect the latent factors in practice. In addition, we showed extensions of our results towards other empirical techniques such as hard negative mining and loss ensembling. We demonstrated that our framework accounts for certain known limitations of \gls{cl}, such as collapsing information, by better capturing style information. At the same time, we showed that there exists a trade-off between recovered style information and downstream classification. Further, even our relaxed assumptions can still be too strong in real-world problems: For example, the positive conditional distribution can be non-Gaussian.
    
    Exploring the gap between theory and practice led to a new loss formulation. We hope that further investigations into overcoming the remaining gap between theory and practice will result in refinements and insights that can improve practical self-supervised learning techniques and our understanding of them. We discussed some directions in \cref{sec:analysis} and \cref{app:analysis}, suggesting that more theoretical work on the identifiability of large-scale DGPs is a fruitful endeavor for future work.

\section*{Author Contributions}
    WB conceived the project idea. ER led the project. ER, RSZ, and WB designed the experiments and ER executed them. ER conducted the analysis with support from RSZ and PR. ER and RSZ jointly created all figures. PR conceived how to incorporate anisotropy in SimCLR. This idea led to WB and RSZ developing the main theorem.
    AJ developed corollaries for hard negative mining and loss ensembling with help from PR and feedback from RSZ. 
    RSZ and WB jointly supervised the project.
    AJ was responsible for presenting the theoretical results in a standardized way.
    ER, PR, AJ, RSZ, and WB contributed to writing the manuscript.

\section*{Acknowledgements}
    The authors thank Julian Bitterwolf, Jack Brady, Steffen Schneider, Thaddäus Wiedemer, and Zac Cranko for helpful discussions. We would also like to thank Mih\'{a}ly Weiner for all the insights that aided us in filling edge cases of our proof.
    The authors thank the International Max Planck Research School for Intelligent Systems (IMPRS-IS) for supporting Evgenia Rusak, Patrik Reizinger, Attila Juhos, and Roland S. Zimmermann. Patrik Reizinger acknowledges his membership in the European Laboratory for Learning and Intelligent Systems (ELLIS) PhD program.
    This work was supported by the German Federal Ministry of Education and Research (BMBF): Tübingen AI Center, FKZ: 01IS18039A. WB acknowledges financial support via an Emmy Noether Grant funded by the German Research Foundation (DFG) under grant no. BR 6382/1-1 and via the Open Philantropy Foundation funded by the Good Ventures Foundation. WB is a member of the Machine Learning Cluster of Excellence, EXC number 2064/1 – Project number 390727645. This research utilized compute resources at the Tübingen Machine Learning Cloud, DFG FKZ INST 37/1057-1 FUGG.

%% file: Tables/augmentations_readout.tex
\captionsetup{aboveskip=10pt,belowskip=-5pt}
\begin{table*}[t]
\setlength{\tabcolsep}{2pt}
\footnotesize
\centering
\caption{
    \textbf{Better Augmentations Readout Does Not Imply Better Classification Performance.}
    Comparison of different training objectives for linear downstream classification (\emph{Classes}) and augmentation readout (\emph{Grayscale} to \emph{Saturate}, \emph{Avg.} denotes the average over those) performance. Learning \gls{diagtemp} results in higher readout accuracy of the used augmentations, which can be interpreted as a better recovery of style latents. However, better augmentations readout does not result in better linear classification accuracy. We note that our theory and all synthetic and VAE experiments apply to the post-projector (backb.+proj.) features which are typically not used in practice because using backbone features results in better performance.
}
 \label{tab:augmentations_readout}

\newcolumntype{h}{>{\columncolor[gray]{0.95}}c}
\newcolumntype{i}{>{\columncolor[gray]{0.85}}c}

\begin{NiceTabular}{c l c c h h h h h i}[colortbl-like]
\toprule
& Training Objective & Features & Classes & Grayscale & Brightness & Contrast & Hue & Saturate & Avg.\\
\midrule
\multirow{6}{*}[-0.4ex]{\rotatebox{90}{CIFAR10}}
& InfoNCE & backbone & 90.9 & 56.1 & 59.3 & 62.6 & 55.4 & 55.1 & 57.7 \\ 
& \ourloss & backbone & 88.4 & 80.1 & 93.4 & 91.1 & 65.5 & 72.3 & 80.5 \\ 
& \ourloss, w/o $\ell_2$-norm. & backbone & 83.2 & 96.2 & 98.0 & 96.6 & 76.8 & 88.4 & 91.2 \\ 
\cmidrule{2-10}
& InfoNCE & backb.+proj. & 89.8 & 52.0 & 52.0 & 52.0 & 51.8 & 51.0 & 51.8 \\ 
& \ourloss & backb.+proj. & 85.4 & 71.4 & 71.4 & 77.1 & 59.7 & 65.5 & 69.0 \\ 
& \ourloss, w/o $\ell_2$-norm. & backb.+proj. & 79.1 & 65.5 & 65.5 & 87.6 & 70.5 & 83.8 & 74.6 \\ 
\midrule
\multirow{4}{*}[-0.4ex]{\rotatebox[origin=c]{90}{ImageNet}}
&  InfoNCE & backbone & 68.2 & 98.7 & 71.6 & 72.5 & 68.0 & 61.2 & 74.4 \\
& \ourloss & backbone & 59.0 & 98.8 & 81.5 & 82.3 & 71.6 & 62.5 & 79.3 \\
\cmidrule{2-10}
& InfoNCE & backb.+proj.& 57.7 & 58.6 & 51.9 & 51.9 & 51.5 & 51.1 & 53.0 \\ 
& \ourloss &backb.+proj.& 26.7 & 71.0 & 55.2 & 54.1 & 52.6 & 52.7 & 57.1 \\ 
\bottomrule
\end{NiceTabular}

\end{table*}

%% file: Tables/causal_3dident_normalized.tex
\captionsetup{aboveskip=10pt,belowskip=10pt}
\begin{table*}[t]
\footnotesize
\centering
\caption{
    \textbf{We Observe the Accuracy-Identifiability Trade-Off on C-3DIdent.}
When using augmentations to generate views (Aug.), InfoNCE achieves better accuracy, while \ourloss has higher $R^2$ scores. When using the ground truth DGP to generate views (DGP), \ourloss outperforms InfoNCE in terms of both accuracy and identifiability scores. We report averaged results over three random seeds and the standard deviation.
}
 \label{tab:3dident}

\newcolumntype{h}{>{\columncolor[gray]{0.9}}c}
\newcolumntype{i}{>{\columncolor[gray]{0.95}}c}

\begin{NiceTabular}{l ccc  ccc }[colortbl-like]
\toprule
& \multicolumn{3}{i}{Data Augmentation}& \multicolumn{3}{i}{DGP} \\
Loss & Class & Lin. id. [$R^2$] & Nonlin. id. [$R^2$] & Class & Lin. id. [$R^2$] & Nonlin. id. [$R^2$] \\
 \midrule
   InfoNCE & \textbf{1.0{\fontsize{1.5}{4}\selectfont $\pm0.0$}} & 0.31{\fontsize{1.5}{4}\selectfont $\pm0.0$} & 0.51{\fontsize{1.5}{4}\selectfont $\pm0.0$} & 1.0{\fontsize{1.5}{4}\selectfont $\pm0.0$} & 0.2{\fontsize{1.5}{4}\selectfont $\pm0.0$} & 0.24{\fontsize{1.5}{4}\selectfont $\pm0.005$} \\
   \ourloss & 0.84{\fontsize{1.5}{4}\selectfont $\pm0.012$} & \textbf{0.34{\fontsize{1.5}{4}\selectfont $\pm0.008$}} & \textbf{0.65{\fontsize{1.5}{4}\selectfont $\pm0.0$}} & \textbf{1.0{\fontsize{1.5}{4}\selectfont $\pm0.0$}}& \textbf{0.38{\fontsize{1.5}{4}\selectfont $\pm0.017$}} & \textbf{0.94{\fontsize{1.5}{4}\selectfont $\pm0.0$}} \\
\bottomrule
\end{NiceTabular}

\end{table*}

%% file: Tables/causal_3dident_prev.tex
\captionsetup{aboveskip=10pt,belowskip=10pt}
\begin{table*}[t]
\setlength{\tabcolsep}{2pt}
\footnotesize
\centering
\caption{
    \textbf{There is no accuracy-identifiability trade-off, if a single augmentation is used for training.} 
On Causal 3DIdent, classification and identifiability results improve simultaneously when switching from InfoNCE to \ourloss, if using a single data augmentation to generate views (\textbf{DA}). The same is observed when perturbing single components in the latent code (\textbf{DGP}). 
This relative advantage of \ourloss compared to InfoNCE was unlocked by giving up sophisticated augmentation pipelines and, hence, absolute performance.
}
 \label{tab:3dident_single}

\newcolumntype{h}{>{\columncolor[gray]{0.9}}c}
\newcolumntype{i}{>{\columncolor[gray]{0.95}}c}

\begin{NiceTabular}{l ccc  ccc ccc}[colortbl-like]
\toprule
& \multicolumn{3}{i}{DA: Small Crops}& \multicolumn{3}{i}{DA: Large Crops}& \multicolumn{3}{i}{DGP: Change Position} \\
Loss & Class & Lin. id. [$R^2$] & Nonlin. id. [$R^2$] & Class & Lin. id. [$R^2$] & Nonlin. id. [$R^2$] & Class & Lin. id. [$R^2$] & Nonlin. id. [$R^2$]\\
 \midrule
   InfoNCE & 0.14 & 0.02 & 0.10 & 0.26 & 0.07 & 0.23 & \textbf{1.0} & 0.25& 0.28 \\
   AnInfoNCE & \textbf{0.37} & \textbf{0.17} &\textbf{0.44}&\textbf{0.47}&\textbf{0.18}&\textbf{0.41}&\textbf{1.0}&\textbf{0.31}&\textbf{0.91} \\
\bottomrule
\end{NiceTabular}

\vspace{3mm}

\begin{NiceTabular}{l ccc  ccc}[colortbl-like]
& \multicolumn{3}{i}{DA: Rotations}&  \multicolumn{3}{i}{DGP: Rotations} \\
Loss & Class & Lin. id. [$R^2$] & Nonlin. id. [$R^2$] & Class & Lin. id. [$R^2$] & Nonlin. id. [$R^2$] \\
 \midrule
   InfoNCE & 0.29 & 0.08&0.30&\textbf{1.0}&0.33&0.45\\
   AnInfoNCE &\textbf{0.48}&\textbf{0.15}&\textbf{0.31}&\textbf{1.0}&\textbf{0.57}&\textbf{0.86}\\
\bottomrule
\end{NiceTabular}

\vspace{3mm}

\begin{NiceTabular}{l ccc  ccc}[colortbl-like]
& \multicolumn{3}{i}{DA: Apply ColorJitter}&  \multicolumn{3}{i}{DGP: Change Hues} \\
Loss & Class & Lin. id. [$R^2$] & Nonlin. id. [$R^2$] & Class & Lin. id. [$R^2$] & Nonlin. id. [$R^2$] \\
 \midrule
   InfoNCE & 0.90&0.22&0.41&\textbf{1.0}&0.31&0.36 \\
   AnInfoNCE &\textbf{0.98}&\textbf{0.56}&\textbf{0.85}&\textbf{1.0}&\textbf{0.44}&\textbf{0.85}\\
\bottomrule
\end{NiceTabular}

\end{table*}

%% file: appendix.tex
\section{Partitioning the Latent Space}
\label{sec:app_generalized_content}
\citet{zimmermann_contrastive_2021} assumes that the latent variables are interchangeable for both their marginals and conditionals are the same. This gives rise to specific symmetries in the latent space \gls{Latent}, which might become restrictive in real-world scenarios. \citet{von_kugelgen_self-supervised_2021} introduces a two-fold partitioning of \gls{Latent}, where the marginal is assumed to be the same for all components \latenti of the latent vector \gls{latent}. Namely, they assume that some latent factors (termed content variables \gls{content}) have a delta distribution as conditional, \ie, those variables do not change across the positive pairs. This model has been deemed useful in understanding why dimensionality collapse happens. However, such an assumption might be too simplistic in specific scenarios (at least for some of the latents), which depends on the augmentation strategy~\cite{dubois_improving_2022,ziyin_loss_nodate,tian_understanding_2022,chen_intriguing_2021,tamkin_feature_nodate,wen_toward_2021,zhai_understanding_2023}. Thus, we can think of \citet{jing_understanding_2022} as a further generalization, though the authors do not make the connection explicit. They consider augmentations, but their reasoning concerns the observation space \gls{Obs}. They conclude---similar to~\cite{dubois_improving_2022,tamkin_feature_nodate}---that the stronger the augmentation, the more prevalent the collapse. Both to generalize and to formalize these insights, we introduce the notion of partitioning of \gls{Latent}:
\begin{definition}[Partitions of \gls{Latent}]
    Given a latent space \gls{Latent} and a latent vector $\gls{latent}\in\gls{Latent},$ we say that subsets of latent components $\latenti^k, \dots, \latenti[i+n-1]^k$ belong to partition $\gls{Latent}^k$ of dimension $n$, when in the ground truth \gls{dgp} their conditional variances $\sigma^2_{\tilde{\gls{latentcomp}}_i|\gls{latent}}$ are the same, \ie:
    \begin{align}
        \forall j\neq l \in \braces{i, \dots, i+n-1} : \sigma^2_{{\gls{latentcomp}}_j^{+}|\gls{latent}} = \sigma^2_{{\gls{latentcomp}}_l^{+}|\gls{latent}}.
    \end{align}
\end{definition}
When all conditional variances are equal, then we get back the setting of \citet{zimmermann_contrastive_2021}; when the number of partitions is two, with one of the conditional variances being zero, then we are in the regime of \citet{von_kugelgen_self-supervised_2021}. We formalize this scenario as follows:
\begin{definition}[Content--style partitioning of \gls{Latent}]\label{definition:content_style_partitioning}
    Following \citet{von_kugelgen_self-supervised_2021}, the latent space \gls{Latent} is said to be partitioned into content (\gls{content}, invariant) and style (\gls{style}, changing) subspaces if and only if
    \begin{enumerate}[nolistsep]
        \item $\gls{Latent}=\gls{Content}\times\gls{Style} : \gls{Content}\cap\gls{Style}=\emptyset; \quad \dim\gls{Content}+\dim\gls{Style}=\dim\gls{Latent}$;
        \item the conditional distribution of \gls{content} latents is a $\delta$-distribution~\citep[Assum.~3.1]{von_kugelgen_self-supervised_2021};
        \item the conditional distribution of \gls{style} latents is non-degenerate~\citep[Assum.~3.2]{von_kugelgen_self-supervised_2021};
    \end{enumerate}
    yielding the following factorization of the latent conditional $\clconditional=\delta\parenthesis{{\gls{latent}^{+,c}}-\gls{content}}p({\gls{latent}^{+,s}}|\gls{style}).$
\end{definition}

That is, our \gls{dgp} does not cover the exact setting of \citet{von_kugelgen_self-supervised_2021}, as we cannot have zero variance. If having zero variance (or equivalently, infinite concentration in \gls{diagtemp}), then we might say that our results incorporate the content-style setting.

\section{Contrastive Learning: Bayes Optimum}\label{sec:CL_Bayes_optimum}

For the sake of completeness, this section presents contrastive learning in a more general form, alongside a theorem regarding its unconditional (Bayes-) optimum. This theorem is referenced several times across the identifiability proofs in \cref{sec:identifiability_proofs}. Although the content of this section is considered to be known by the research community, we still decided to include it to make the presentation of the theory more self-contained. \cref{sec:identifiability_proofs} presents the novelties introduced by this paper.

\begin{scpcmd}[
\newscpcommand{\domain}{}{\mathcal{Y}}
\newscpcommand{\anchor}{}{\myvec{y}}
\newscpcommand{\anchort}{}{\tilde{\anchor}}
\newscpcommand{\basepair}{}{\anchor}
\newscpcommand{\ppair}{}{\basepair^+}
\newscpcommand{\npair}{[1][]}{\basepair^-_{#1}}
\newscpcommand{\nrnegs}{}{M}
\newscpcommand{\samplevec}{}{\myvec{Y}}
\newscpcommand{\samplevecdist}{}{\dist[\samplevec]}
\newscpcommand{\hypclass}{}{\mathcal{U}}
\newscpcommand{\genclloss}{}{\mathcal{L}_g}
\newscpcommand{\u}{}{u}
\newscpcommand{\expsim}{[2][\anchor]}{e^{\u(#1, #2)}}
\newscpcommand{\negdist}{}{\dist[\mathrm{neg}]}
\newscpcommand{\posdist}{}{\dist[\mathrm{pos}]}
\newscpcommand{\K}{}{K}
\newscpcommand{\pair}{[1]}{\tilde{\basepair}_{#1}}
\newscpcommand{\idloss}{}{\mathcal{L}^{id}}
\newscpcommand{\dirac}{[1]}{\delta_{\K=#1}}
]

\begin{remark}
    In the following, we will refer to vectors as $\anchor$ to provide a unified treatment both for $\anchor=\gls{latent}$ and $\anchor=\gls{obs}.$
\end{remark}
    
\begin{definition}[General Contrastive Learning (CL) Problem]\label{def:gen_CL}
    Let $\domain$ be a data domain. Let $\samplevec \defeq (\anchor, \ppair, \npair[1], \ldots, \npair[\nrnegs]) \sim \samplevecdist$ be a random vector with all elements coming from $\domain$. We refer to $\anchor$ as the \emph{anchor point}, to $\ppair$ as its \emph{positive pair} and to $\npair[i]$ as one of its \emph{negative pairs}. Let us assume that
    \begin{enumerate}
        \item the pairs $\ppair, \npair[1], \ldots, \npair[\nrnegs]$ are conditionally completely independent when conditioned on $\anchor$ and
        \item\label{enumitem:gen_cl_equal_negative_conditionals} the conditional distribution of negative pairs \wrt the anchor coincide, \ie, for any $i,j$ we have $\dist[\npair[i]|\anchor] = \dist[\npair[j]|\anchor]$.
    \end{enumerate}
    Let $\hypclass$ be a class of bivariate functions $\u:\domain\times\domain\rightarrow\rr$, hereinafter called \emph{similarity functions} and let us optimize the following objective amongst all functions $\u \in \hypclass$:
    \begin{equation}
        \min_{\u\in\hypclass} \quad \genclloss(u) \defeq \expectation \bigg[ -\ln \frac{\expsim{\ppair}}{\expsim{\ppair} + \sum_{j=1}^\nrnegs \expsim{\npair[j]}} \bigg].
    \end{equation}
    The tuple $(\domain, \samplevecdist, \hypclass, \genclloss)$ is called a \emph{general contrastive learning (CL) problem}.
\end{definition}

\begin{remark}\label{rem:data_distro_gen_CL}
    Denoting $\posdist\defeq\dist[\ppair|\anchor]$ and $\negdist \defeq \dist[\npair[j]|\anchor]$, the latter being independent of the exact choice of $j$ per assumption~\ref{enumitem:gen_cl_equal_negative_conditionals}, the joint data distribution is the following:
    \begin{equation}
        \samplevecdist(\anchor,\ppair,\npair[1],\ldots,\npair[\nrnegs]) = \dist[\anchor](\anchor) \, \posdist(\ppair|\anchor) \, \prod_{j=1}^{\nrnegs} \negdist(\npair[j]|\anchor).\label{eq:data_distro_gen_CL_joint}
    \end{equation}
\end{remark}

\begin{theorem}[Bayes-Optima of CL]\label{theo:bayes_optima_cl}
    Let $(\domain,\samplevecdist,\hypclass,\genclloss)$ be a general CL problem (\cref{def:gen_CL}) with $\samplevec \defeq (\anchor, \ppair, \npair[1], \ldots, \npair[\nrnegs]) \sim \samplevecdist$ a \emph{continuous random vector}, such that:
    \begin{enumerate}
        \item \label{enumitem:achor_support} $\anchor$ is supported on $\domain$,
        \item \label{enumitem:conditionals_support} for any $\anchor\in\domain$, both conditional distributions $\negdist(\cdot|\anchor),\posdist(\cdot|\anchor)\defeq\dist[\ppair|\anchor](\cdot|\anchor)$ are supported on $\domain$.
    \end{enumerate}
    Then an arbitrary measurable function $\u:\domain\times\domain\rightarrow\rr$ is a global minimum of $\genclloss$ (amongst all measurable $u$'s) if and only if there exists a measurable function $c:\domain\rightarrow\rr$ such that the following holds almost everywhere (\wrt any continuous measure of $\domain$):
    \begin{equation}
        \u(\anchor,\anchort) = c(\anchor) + \ln\frac{\posdist(\anchort|\anchor)}{\negdist(\anchort|\anchor)} .
    \end{equation}
\end{theorem}

This statement has already been proven by \citet{matthes2023towards}, however we provide an alternative argument. The key ideas have also emerged in \citet{sohn2016improved, gutmann_noise-contrastive_nodate, ma2018noise}.

To prove the theorem, we rewrite the optimization objective as the classification objective of discriminating the positive pair from all the negative pairs, given an anchor point. More specifically, all the pairs $\ppair, \npair[1], \ldots, \npair[\nrnegs]$ are randomly shuffled, and, then, the task is to estimate the index $\K$ of the positive pair from the new sequence, $\pair{0}, \ldots, \pair{\nrnegs}$, given the unchanged anchor $\anchor$. The optimization objective is a categorical \gls{ce} with a predictor of the form (being the output of a softmax function, the predictor's output is normalized to a finite probability distribution):
\begin{equation}
    \bigg[ \frac{\expsim{\pair{i}}}{\sum_j \expsim{\pair{j}}}  \bigg|\; i\in\{0,1,\ldots,\nrnegs\}\bigg].
\end{equation}
We formalize the instance discrimination task as follows:

\begin{definition}[Instance Discrimination Data]\label{def:instance_discrimination_data}
    The \emph{instance discrimination data} $(\anchor, \pair{0}, \ldots, \pair{\nrnegs}; \K)$ corresponding to a general \gls{cl} problem $(\domain, \samplevecdist, \hypclass, \genclloss)$ is defined by the following process:
    \begin{enumerate}[label=\textit{Step \arabic*}),leftmargin=1.5cm]
        \item $\anchor \sim \dist[\anchor]$
        \item \label{enumitem:K_uniform} $K \sim Uni(\{0,1,\ldots,\nrnegs\})$
        \item $\pair{\K}\sim\posdist(\cdot|\anchor)$
        \item $\forall j \in \{\,0,1,\ldots,K-1, K+1,\ldots, \nrnegs\,\}: \pair{j}\sim\negdist(\cdot|\anchor)$.        
    \end{enumerate}
\end{definition}

\begin{remark}
    The joint data distribution of the instance discrimination data is:
    \begin{equation}
        \dist(\anchor,\pair{0},\ldots,\pair{\nrnegs},\K=i) = \dist[\anchor](\anchor)\,\prob(\K=i)\,\posdist(\pair{i}|\anchor)\prod_{0 \leq j \leq \nrnegs,\, j \neq i} \negdist(\pair{j}|\anchor). \label{eq:ins_disc_data_dist}
    \end{equation}
\end{remark}

\begin{definition}[Instance Discrimination Problem]\label{def:instance_discrimination_problem}
    Let the instance discrimination data $(\anchor, \pair{0}, \ldots, \pair{\nrnegs}; \K)$ correspond to a general \gls{cl} problem $(\domain, \samplevecdist, \hypclass, \genclloss)$ (\cref{def:instance_discrimination_data}) that also satisfies the assumptions of \cref{theo:bayes_optima_cl}. We define the \emph{instance discrimination problem} as:
    \begin{equation}
        \min_{\u\in\hypclass} \quad \idloss(u) \defeq \expectation \bigg[ -\sum_{i=0}^\nrnegs \dirac{i}\ln \frac{\expsim{\pair{i}}}{\sum_{j=0}^\nrnegs \expsim{\pair{j}}} \bigg],
    \end{equation}
    where $\delta_E$ is the indicator variable of event $E$.
\end{definition}


\begin{remark}\label{rem:data_distro_instance_discrimination}
    When conditioned on $\K=k$, the data distribution \cref{eq:ins_disc_data_dist} can be slightly simplified to:
    \begin{equation}
        \dist(\anchor,\pair{0},\ldots,\pair{\nrnegs}|\K=k) =
        \dist[\anchor](\anchor) \, \posdist(\pair{k}|\anchor) \, \prod_{0 \leq j \leq \nrnegs,\, j \neq k} \negdist(\pair{j}|\anchor).\label{eq:ins_disc_data_dist_cond}
    \end{equation}
    Comparing \cref{eq:ins_disc_data_dist_cond} to \cref{eq:data_distro_gen_CL_joint} in \cref{rem:data_distro_gen_CL}, this distribution acts as if $\pair{k}$ (or $\pair{j\neq k}$) was a positive pair (or negative pairs) in a contrastive learning problem \wrt anchor $\anchor$.
\end{remark}

\begin{lemma}\label{lem:idloss_eq_genclloss}
    Under the assumptions of \cref{def:instance_discrimination_problem} we have $\genclloss(\u) = \idloss(\u)$, for any $u \in \hypclass$.
\end{lemma}
\begin{proof}
    First, using the law of total expectation $\expectation{[A]} = \expectation{[\expectation{[A|B]}]}$, we expand $\idloss$ \wrt $\K$:
    \begin{align}\label{eq:instance_disc_loss_wrt_K}
        \idloss(\u) &= \expectation[\K] \Bigg[\, \expectation \bigg[ -\sum_{i=0}^\nrnegs \dirac{i}\ln \frac{\expsim{\pair{i}}}{\sum_{j=0}^\nrnegs \expsim{\pair{j}}} \bigg| \K \bigg]\, \Bigg] \\
        &= \frac{1}{\nrnegs+1} \mathlarger{\sum}_{k=0}^{\nrnegs} \expectation \bigg[ -\sum_{i=0}^\nrnegs \dirac{i}\ln \frac{\expsim{\pair{i}}}{\sum_{j=0}^\nrnegs \expsim{\pair{j}}} \bigg| \K=k \bigg].
    \end{align}
    However, for the individual summands:
    \begin{equation}
        \expectation \bigg[ -\sum_{i=0}^\nrnegs \dirac{i}\ln \frac{\expsim{\pair{i}}}{\sum_{j=0}^\nrnegs \expsim{\pair{j}}} \bigg| \K=k \bigg] =
        \expectation \bigg[ -\ln\frac{\expsim{\pair{k}}}{\sum_{j=0}^\nrnegs \expsim{\pair{j}}} \bigg| \K=k \bigg]
    \end{equation}
    Considering \cref{rem:data_distro_instance_discrimination}, the last term equals the contrastive objective $\genclloss(\u)$.

    Therefore, substituting into \cref{eq:instance_disc_loss_wrt_K} we get
    \begin{equation}
        \idloss(\u) = \frac{1}{\nrnegs+1} \sum_{k=0}^{\nrnegs} \genclloss(u) = \genclloss(u).
    \end{equation}
\end{proof}

After having rewritten the general contrastive objective, we are ready to prove \cref{theo:bayes_optima_cl}.


\begin{scpcmd}[
\newscpcommand{\pairvec}{}{\tilde{\samplevec}}
\newscpcommand{\condprob}{}{\eta}
\newscpcommand{\condprobvec}{}{\myvec{\condprob}}
\newscpcommand{\modprob}{}{F}
\newscpcommand{\modprobvec}{}{\myvec{\modprob}}
]

\begin{proof}[Proof of \cref{theo:bayes_optima_cl}]
    By taking the corresponding instance discrimination problem $\idloss$, it remains to compute the optima of $\idloss$. For this, we again expand $\idloss$ with the law of total expectation, but now \wrt $(\anchor,\pair{0},\ldots,\pair{\nrnegs})$:
    \begin{align}
        \idloss(\u) &= \expectation \Bigg[\, \expectation \bigg[ -\sum_{i=0}^\nrnegs \dirac{i}\ln \frac{\expsim{\pair{i}}}{\sum_{j=0}^\nrnegs \expsim{\pair{j}}} \bigg| \anchor,\pair{0},\ldots,\pair{\nrnegs} \bigg] \,\Bigg] \\
        &= \expectation \bigg[ -\sum_{i=0}^{\nrnegs} \expectation \Big[ \dirac{i} \Big| \anchor,\pair{0},\ldots,\pair{\nrnegs} \Big] \ln \frac{\expsim{\pair{i}}}{\sum_{j=0}^\nrnegs \expsim{\pair{j}}} \bigg] \\
        &= \expectation \bigg[ -\sum_{i=0}^{\nrnegs} \prob \Big( \K = i \Big| \anchor,\pair{0},\ldots,\pair{\nrnegs} \Big) \ln \frac{\expsim{\pair{i}}}{\sum_{j=0}^\nrnegs \expsim{\pair{j}}} \bigg],
    \end{align}
    where we used the properties that $\expectation{\big[A \, f(B)\big|B\big]} = f(B)\,\expectation{\big[A\big|B\big]}$ and that for any event $E$, $\expectation{[\delta_E]} = \prob(E)$.

    Let us denote $\pairvec = ( \anchor,\pair{0},\ldots,\pair{\nrnegs} )$ and $\condprob_i\big(\pairvec\big) \defeq \prob \big( \K = i \big| \anchor,\pair{0},\ldots,\pair{\nrnegs} \big)$, the quantity $\condprobvec\big(\pairvec\big) = (\condprob_0, \ldots,\condprob_\nrnegs)$ is the ground-truth conditional distribution of index $\K$. Denoting $\modprob_i \big(\pairvec\big) = \expsim{\pair{i}}/(\sum_{j=0}^\nrnegs \expsim{\pair{j}})$, the quantity $\modprobvec\big(\pairvec\big) = (\modprob_0, \ldots, \modprob_\nrnegs)$ is the inferred conditional distribution of index $\K$. Then,

    \begin{align}
        \idloss(\u) &= \expectation_{\pairvec} \bigg[ -\sum_{i=0}^\nrnegs \condprob_i\big(\pairvec\big) \ln \modprob_i\big(\pairvec\big) \bigg]\\
        &= \expectation_{\pairvec} \bigg[ \sum_{i=0}^\nrnegs \condprob_i\big(\pairvec\big) \ln \frac{\condprob_i\big(\pairvec\big)}{\modprob_i\big(\pairvec\big)} \bigg] + \expectation_{\pairvec} \bigg[ -\sum_{i=0}^\nrnegs \condprob_i\big(\pairvec\big) \ln \condprob_i\big(\pairvec\big) \bigg] \\
        &= \expectation_{\pairvec} {\Big[\,{KL}\Big(\condprobvec\big(\pairvec\big) \Big\| \modprobvec\big(\pairvec\big)\Big)\,\Big]} + \expectation_{\pairvec} {\Big[\,{H}\Big(\condprobvec\big(\pairvec\big)\Big)\,\Big]},
    \end{align}
    where $KL$ and $H$ denote the \gls{kld} and entropy of the finite distributions over indices $\{0,1,\ldots,\nrnegs\}$, provided as arguments.
    Due to the non-negativity of the \gls{kld}, we can lower bound $\genclloss(\u)$:
    \begin{equation}
        \genclloss(\u) = \idloss(\u) \geq \expectation_{\pairvec} {\Big[\,{H}\Big(\condprobvec\big(\pairvec\big)\Big)\,\Big]},
    \end{equation}
    which is independent of the choice of the function $\u$. Equality is attained if and only if the \gls{kld} vanishes, which occurs if and only if the true and inferred conditional distribution of index $\K$
    coincide almost everywhere \wrt $\pairvec\sim\dist(\anchor,\pair{0},\ldots,\pair{\nrnegs})$. The statement is in terms of the indices, since the instance discrimination task is about classifying which index belongs to the positive pair. Hence, for any $i$ and almost everywhere \wrt $\pairvec$:
    \begin{equation}\label{eq:true_mod_conditional_equal1}
        \frac{\expsim{\pair{i}}}{\sum_{j=0}^\nrnegs \expsim{\pair{j}}} = \modprob_i (\pairvec) = \condprob_i(\pairvec) = \prob \big( \K = i \big| \anchor,\pair{0},\ldots,\pair{\nrnegs} \big).
    \end{equation}
    To rewrite the last expression, we use assumptions \ref{enumitem:achor_support} and \ref{enumitem:conditionals_support} and the uniformity of $\K$ (step~\ref{enumitem:K_uniform}). Thus, the data distribution \cref{eq:ins_disc_data_dist} of the instance discrimination problem can be rewritten as
    \begin{equation}
        \dist(\anchor,\K=i,\pair{0},\ldots,\pair{\nrnegs}) = \frac{1}{\nrnegs + 1}\,\dist[\anchor](\anchor)\,\frac{\posdist(\pair{i}|\anchor)}{\negdist(\pair{i}|\anchor)}\prod_{j=0}^\nrnegs \negdist(\pair{j}|\anchor).
    \end{equation}
    With this in mind, \cref{eq:true_mod_conditional_equal1} becomes
    \begin{equation}
        \frac{\expsim{\pair{i}}}{\sum_{j=0}^\nrnegs \expsim{\pair{j}}} = \prob \big( \K \!=\! i \big| \anchor,\pair{0},\ldots,\pair{\nrnegs} \big) = \! \frac{\dist(\anchor,\K\!=\!i,\pair{0},\ldots,\pair{\nrnegs})}{\sum_{j=0}^{\nrnegs} \dist(\anchor,\K\!=\!j,\pair{0},\ldots,\pair{\nrnegs})} = \frac{\frac{\posdist(\pair{i}|\anchor)}{\negdist(\pair{i}|\anchor)}}{\sum_{j=0}^\nrnegs\frac{\posdist(\pair{j}|\anchor)}{\negdist(\pair{j}|\anchor)}}.
    \end{equation}
    From this it can be seen that there exists a measurable, positive function $\tilde{c}:\domain \rightarrow \rr$ such that
    \begin{equation}
        \expsim{\anchort} = \tilde{c}(\anchor) \cdot \frac{\posdist(\anchort|\anchor)}{\negdist(\anchort|\anchor)} \quad\text{holds almost everywhere (\wrt any continuous measure of $\domain$).}
    \end{equation}
    By taking the logarithm we get the desired result.
\end{proof}

\end{scpcmd}

\end{scpcmd}

\section{Identifiability Theory}\label{sec:identifiability_proofs}

\subsection{Identifiability of Anisotropic \gls{cl}}\label{subsec:ident_base}
\begin{scpcmd}[
\newscpcommand{\d}{}{d}
\newscpcommand{\D}{}{D}
\newscpcommand{\Z}{}{\mathcal{Z}}
\newscpcommand{\S}{[1][\d-1]}{\mathbb{S}^{#1}}
\newscpcommand{\g}{}{\myvec{g}}
\newscpcommand{\f}{}{\myvec{f}}
\newscpcommand{\X}{}{\mathcal{X}}
\newscpcommand{\z}{}{\myvec{z}}
\newscpcommand{\pz}{}{\z^{+}}
\newscpcommand{\nz}{[1]}{\z^{-}_{#1}}
\newscpcommand{\x}{}{\myvec{x}}
\newscpcommand{\px}{}{\x^{+}}
\newscpcommand{\nx}{[1]}{\x^{-}_{#1}}
\newscpcommand{\distancesquare}{[4][]}{#1\lVert#2-#3#1\rVert^2_{#4}}
\newscpcommand{\anisoexpsim}{[4][]}{e^{-\distancesquare[#1]{#2}{#3}{#4}}}
\newscpcommand{\Lam}{}{\myvec{\Lambda}}
\newscpcommand{\infLam}{}{\myvec{\hat{\Lambda}}}
\newscpcommand{\M}{}{M}
\newscpcommand{\aniclloss}{}{\gls{cllossgen}}
\newscpcommand{\latentloss}{}{\tilde{\mathcal{L}}}
\newscpcommand{\ig}{}{\g^{-1}}
\newscpcommand{\h}{}{\myvec{h}}
\newscpcommand{\zt}{}{\tilde{\z}}
]

\begin{definition}[Anisotropic data generating process, ADGP]\label{def:adgp}

    Let $\Z \defeq \S$ be the $(\d-1)$-dimensional unit hypersphere and let $\g:\Z\rightarrow\X \subseteq \rr[\D]$ be an invertible and continuous function. 
    Let $\Lam \in \rr[\d\times\d]$ be a positive definite diagonal matrix.
    
    The following process is called an \emph{anisotropic data generating process (ADGP)}:
    \begin{enumerate}[label=\textit{Step \arabic*}\textdegree,leftmargin=1.5cm]
        \item $\z \sim Uni(\Z)$
        \item $\pz \sim \dist[\pz|\z](\cdot|\z)$, where $\dist[\pz|\z](\pz|\z) \propto\anisoexpsim{\pz}{\z}{\Lam}$
        \item for all $j \in \{1,\ldots,\M\}, \nz{j} \sim Uni(\Z)$
        \item \label{enum:ADP:step_4} $\x, \px, \{ \nx{j}\} \overset{\g}{\longleftarrow} \z, \pz, \{ \nz{j} \}$
    \end{enumerate}
    The pair $(\g,\Lam)$ is called the \emph{characterization} of the process. Sets $\Z, \X$ are called the \emph{latent} and \emph{observation spaces}, respectively.
\end{definition}

\begin{remark}
    The random vector $(\x, \px, \{\nx{j}\})$ meets the assumptions of \cref{def:gen_CL}, so a contrastive learning problem can be well-defined.
\end{remark}

\begin{definition}[\acrlong{aninfonce} loss, \ourloss]\label{def:aclp}
    Let an ADGP be characterized by $(\g,\Lam)$ and let $(\x, \px, \{\nx{j}\})$ be the observed data (all coming from observation space $\X \subseteq \rr[\D])$. Let us optimize the following objective amongst all \gls{pd} diagonal matrices $\infLam$ and all continuous encoders $\f:\rr[\D] \rightarrow \Z$:
    \begin{equation}
        \min_{\substack{\infLam \text{ PD diag}\\ \f \text{ continuous}}} \quad \aniclloss(\f,\infLam) \defeq \expectation \bigg[ -\ln \frac{\anisoexpsim{\f(\px)}{\f(\x)}{\infLam}}{\anisoexpsim{\f(\px)}{\f(\x)}{\infLam} + \sum_{j=1}^\M \anisoexpsim{\f(\nx{j})}{\f(\x)}{\infLam}} \bigg].
    \end{equation}
    We call the objective $\aniclloss$ the \emph{\acrfull{aninfonce} loss}.
\end{definition}

\begin{remark}\label{rem:special_cl_problem}
    The tuple $(\X, \dist(\x, \px, \{\nx{j}\}), \mathcal{U}, \aniclloss)$ is a special case of the general contrastive learning problem (\cref{def:gen_CL}), where $\mathcal{U} = \{ u \,|\, u(\x, \px) = -\| \f(\px) - \f(\x) \|^2_{\infLam} \text{ for any pair } (\f, \infLam) \}$ is the hypothesis class of similarity functions.
\end{remark}

\begin{reptheorem}{prop:min_ce_maintains_weighted_mse}[Identifiability of Anisotropic \gls{cl}]\label{theo:ident_ADGP}
    Let an ADGP (\cref{def:adgp}) be characterized by $(\g,\Lam)$ with latent and observed data being $(\z,\pz,\{\nz{j}\})$ and $(\x,\px,\{\nx{j}\})$
    . Let $\aniclloss$ denote the \ourloss loss (\cref{def:aclp}).
    
    If a pair $(\f, \infLam)$ (globally) minimizes the $\aniclloss$ loss, then:
    \begin{enumerate}
        \item $\infLam$ is equal to $\Lam$ up to a permutation of elements and
        \item $\f\circ\g$ is a block-orthogonal transformation, where each block acts on latents of equal weight $\Lam_{ii}$.
    \end{enumerate}
    In other words, latent $\z$ is identified up to a block-orthogonal transformation.
\end{reptheorem}

\begin{proof}[Proof of Thm.~\ref{theo:ident_ADGP}]
Our proof consists of two steps:
\begin{enumerate}[label=\textbf{Step \arabic*:}, leftmargin=1.25cm]
    \item Based on \cref{sec:CL_Bayes_optimum}, optimums of \ourloss connect the density ratio to the similarity function $u$.
    \item Then, we solve the resulting functional equation.
\end{enumerate}
    \paragraph{Step 1: Connecting the density ratio to the similarity function.}
    First, we are going rewrite the $\aniclloss$ loss function into a form that favors the analysis of $\f\circ\g$. Plugging in $\x = \g(\z)$ (\ref{enum:ADP:step_4} in \cref{def:adgp}) into the $\aniclloss$-loss:
    \begin{equation}
        \aniclloss(\f,\infLam) = \expectation \bigg[ -\ln \frac{\anisoexpsim{\f\circ\g(\pz)}{\f\circ\g(\z)}{\infLam}}{\anisoexpsim{\f\circ\g(\pz)}{\f\circ\g(\z)}{\infLam} + \sum_{j=1}^\M \anisoexpsim{\f\circ\g(\nz{j})}{\f\circ\g(\z)}{\infLam}} \bigg]\eqdef \latentloss(\f\circ\g, \infLam),
    \end{equation}
    with the optimization still over $\f$ (and $\infLam$). However, as the generator $\g$ is continuously invertible on the \emph{compact} set $\Z$, its inverse $\ig$ is automatically continuous as well. Therefore, any continuous function $\h:\Z\rightarrow\Z$ can take the role of $\f\circ\g$, by substituting $\f=\h\circ\ig$ continuous. Hence, minimizing $\aniclloss(\f,\infLam)$ for $\f$ continuous is equivalent to minimizing the new, latent space loss $\latentloss(\h,\infLam)$ for $\h$ continuous:
    \begin{equation}
        \min_{\substack{\infLam \text{ PD diag}\\ \f \text{ cont}}} \aniclloss(\f, \infLam) = \min_{\substack{\infLam \text{ PD diag} \\ \h \text{ cont}}} \latentloss(\h, \infLam).
    \end{equation}
    The new hypothesis class is $\tilde{\mathcal{U}} = \{ u \,|\, u(\z,\pz)=-\distancesquare{\h(\pz)}{\h(\z)}{\infLam} \text{ for any pair } (\h, \infLam)\}$. 
    
    We claim that $\latentloss$ can attain its Bayes- or unconditional optimum (\cref{theo:bayes_optima_cl}), \ie, as if there were no constraints on the similarity function $u$. According to \cref{theo:bayes_optima_cl}, an arbitrary similarity function $u$ minimizes $\latentloss$ amongst all possible measurable $u$'s (attaining the Bayes-optimum) if and only if for a suitable function $c$:
    \begin{equation}\label{eq:equation_Bayes_optimum_ADGP}
        u(\z,\zt) = c(\z) + \ln\frac{\dist[\pz|\z](\zt|\z)}{\dist[\nz{}|\z](\zt|\z)} \quad\text{holds almost everywhere on $\Z$.}
    \end{equation}
    In our special case, $u$ is restricted to take the form $u(\z,\zt) = -\distancesquare{\h(\zt)}{\h(\z)}{\infLam}$. Besides that, $\nz{j}\sim Uni(\Z)$ with a constant density function and $\dist[\nz{}|\z] = \dist[\nz{}] \equiv \text{const}$. We also plug in the anisotropic positive conditional distribution $\dist[\pz|\z](\zt|\z)\propto \anisoexpsim{\zt}{\z}{\Lam}$. After merging every constant additive term (including normalization constants) into \makebox{function $c$}, we get the functional equation with unknowns $\h,\infLam,c$:
    \begin{equation}\label{eq:optimality_special_case}
        \distancesquare{\h(\zt)}{\h(\z)}{\infLam} = c(\z) + \distancesquare{\zt}{\z}{\Lam} \quad\text{holds almost everywhere on $\Z$.}
    \end{equation}
    Observe that the equation is trivially satisfied for $\h=id_\Z, \infLam=\Lam, c\equiv 0$. This special case corresponds to having $\f = \inv{\g}$, \ie, we inverted the exact generator and also reconstructed $\Lam$. Consequently, our loss function attains its Bayes-minimum even in our restricted hypothesis class.    
    Now leveraging again \cref{theo:bayes_optima_cl} in the other direction, the above shows us that all minimizers are, in fact, solutions of \cref{eq:optimality_special_case}.
    
    Let's take an arbitrary solution $\h, \infLam, c$. Noticeably, the single-variable, measurable function $c$ is equal (almost everywhere) to a two-variable continuous function. This is only possible if the latter function, depending on $(\z,\zt)$ is a single-variable, continuous function of $\z$. In this case, swapping $c$ to this exact function will also reveal a good solution and, therefore, it is only required to solve \cref{eq:optimality_special_case} for $c$ continuous and on the full space $\Z$. By taking the special case of $\zt \leftarrow \z$, we see that $c(\z) = 0$ follows.

    \begin{scpcmd}[
    \newscpcommand{\a}{}{\myvec{a}}
    \newscpcommand{\b}{}{\myvec{b}}
    \newscpcommand{\halfpow}{}{{\!\sfrac{1\!}{2}}}
    \newscpcommand{\neghalfpow}{}{{\!\!\sfrac{-\!\!1\!}{2}}}
    \newscpcommand{\O}{}{\myvec{O}}
    \newscpcommand{\H}{}{\myvec{H}}
    \newscpcommand{\c}{}{\myvec{c}}
    ]

    \paragraph{Step 2: Solving the functional equation.}
    What is left to solve is the following functional equation in terms of the variables $\h, \infLam$:
    \begin{equation}\label{eq:}
        \distancesquare{\h(\zt)}{\h(\z)}{\infLam} = \distancesquare{\zt}{\z}{\Lam} \quad\text{holds for any $\zt,\z$ on $\Z$.}
    \end{equation}
    For this, let $\z_0 \in \Z$ be fixed and let us define the following functions:
    \begin{equation}
        \a(\z) = \infLam^\halfpow \big(\h(\z) - \h(\z_0)\big) \quad\text{and}\quad 
        \b(\z) = \Lam^\halfpow (\z-\z_0).
    \end{equation}
    In this case, the following properties are easily verified:
    \begin{gather}
        \distancesquare{\a(\zt)}{\a(\z)}{} = \distancesquare{\b(\zt)}{\b(\z)}{}, \label{eq:iso_ab} \\
        \a(\z_0) = \b(\z_0) = 0. \label{eq:z_0}
    \end{gather}
    By substituting $\zt \leftarrow \z_0$ into \cref{eq:iso_ab} and using \cref{eq:z_0}, we get that
    \begin{equation}
        \normsquared{\a(\z)} = \normsquared{\b(\z)} \text{holds for any $\z \in \Z$.} \label{eq:isonorm}
    \end{equation}
    After expanding \cref{eq:iso_ab} and using \cref{eq:isonorm} twice, we get that, similarly to the distance and norm, the scalar product is equivalently transformed by $\a$ and $\b$:
    \begin{equation}
        \scprod{\a(\zt),\a(\z)} = \scprod{\b(\zt),\b(\z)} \text{ for any $\zt,\z \in \Z$.}\label{eq:equiv_scprod}
    \end{equation}

    From this, it can be proven that there exists an orthogonal linear transformation $\O$ of $\rr[\d]$ such that:
    \begin{equation}
        \a(\z) = \O\b(\z) \text{ holds for any $\z \in \Z$.}
    \end{equation}
    To see this, we first prove that the following mapping is well-defined on the linear span of $\b(\Z)$:
    \begin{equation}
        \O: \sum_{i=1}^n \alpha_i \b(\z_i) \mapsto \sum_{i=1}^n \alpha_i \a(\z_i) \label{eq:O_initial_def}
    \end{equation}
    Assume there exist two equal expansions (with potentially zero coefficients):
    \begin{equation}
        \sum_{i=1}^n \alpha_i \b(\z_i) = \sum_{i=1}^n \beta_i \b(\z_i).
    \end{equation}
    In this case the difference and the distance square vanish:
    \begin{equation}
        \bigg\Vert{\sum_{i=1}^n (\alpha_i - \beta_i) \b(\z_i)}\bigg\Vert^2 = 0. \label{eq:diff_vanish_b}
    \end{equation}
    However, due to the equivalent transformation of the scalar product (\cref{eq:equiv_scprod}), after the expansion of \cref{eq:diff_vanish_b} we can essentially change every occurrence of $\b$ to $\a$ and receive that:
    \begin{equation}
        \bigg\Vert{\sum_{i=1}^n (\alpha_i - \beta_i) \a(\z_i)}\bigg\Vert^2 = 0, \label{eq:diff_vanish_a}
    \end{equation}
    which may hold if and only if
    \begin{equation}
        \sum_{i=1}^n \alpha_i \a(\z_i) = \sum_{i=1}^n \beta_i \a(\z_i),
    \end{equation}
    proving that $\O$ in \cref{eq:O_initial_def} is well-defined on the linear span of $\b(\Z)$.
    
    Secondly, the facts that $\O$ is linear and that $\a(\z) = \O\b(\z)$ hold trivially. We also see that the image of $\b$, \ie $\b(\Z) = \b(\S)$ is full dimensional and hence $\O$ is defined on the entire $\rr[\d]$. The orthogonality is a direct consequence of the scalar product preservation property (\cref{eq:equiv_scprod}) and the definition of $\O$ (\cref{eq:O_initial_def}).

    Consequently, we have proven that there exists an orthogonal linear transformation $\O$ of $\rr[\d]$ such that:
    \begin{equation}
        \infLam^\halfpow (\h(\z)-\h(\z_0)) = \O\Lam^\halfpow (\z-\z_0) \text{ holds for any $\z\in\Z$.}
    \end{equation}
    As $\infLam$ is positive definite and invertible, it follows (after merging constant terms containing $\z_0$), that:
    \begin{equation}
        \h(\z) = \infLam^\neghalfpow \O \Lam^\halfpow \z + \c \text{ holds for any $\z \in \Z$, for some constant $c$.}
    \end{equation}
    Let us denote $\H = \infLam^\neghalfpow \O \Lam^\halfpow$. To complete the proof, we have to show that $\c=0$ and $\H$ is orthogonal. For this, we use the fact that $\h$ is a transformation of the unit hypersphere $\Z = \S$:
    \begin{align}
        \scprod{\H\z, \c} &= \frac{1}{4} \Big( \Vert \H\z + \c \Vert^2 - \Vert -\!\H\z + \c \Vert^2 \Big) \\
        &= \frac{1}{4} \Big( \Vert \h(\z) \Vert^2 - \Vert \h(-\z) \Vert^2 \Big) = \frac{1}{4} (1-1) = 0\text{, for any $\z \in \Z$.}
    \end{align}
    As $\H$ is invertible and, thus, with a full range of $\rr[\d]$, $\c=0$ is concluded. Therefore, $\h$ is linear. As unit vectors are mapped to unit vectors, $\h$ is norm-preserving and, thus, $\h = \f\circ\g$ is orthogonal.

    Now, rewriting $\H = \infLam^\neghalfpow \O \Lam^\halfpow$ gives us:
    \begin{equation}
        \O\Lam^\halfpow = \infLam^\halfpow\H = \H \big( \H^\top \infLam^\halfpow\H \big),
    \end{equation}
    where the first and last formulas provide polar decompositions of the same operator. As the polar decomposition of invertible operators is unique, we conclude that $\O = \H$, and
    \begin{equation}
        \Lam^\halfpow = \H^\top \infLam^\halfpow\H.
    \end{equation}

    After squaring both sides, we get:
    \begin{align}
        \Lam &= \H^\top \infLam\H.\label{eq:eigen}\\
        \intertext{Then the right-hand side of \cref{eq:eigen} is the eigendecomposition of $\Lam$. Due to the uniqueness of eigenvalues, there must exist a permutation $\gls{permutation}_{\h}$ such that $\infLam = \gls{permutation}_{\h}\Lam\gls{permutation}_{\h}^\top$. Then}
        \Lam &= \H^\top \gls{permutation}_{\h}\Lam\gls{permutation}_{\h}^\top\H\\
        \Leftrightarrow \gls{permutation}_{\h}^\top\H\Lam &= \Lam\gls{permutation}_{\h}^\top\H.
    \end{align}
    In other words, $\gls{permutation}_{\h}^\top\H$ and $\Lam$ commute. Then $\gls{permutation}_{\h}^\top\H$ must be a diagonal matrix if all values on the diagonal of $\Lam$ are distinct. Should some values be the same, then the corresponding submatrix of $\gls{permutation}_{\h}^\top\H$ is orthogonal.
    
    To see this more formally, we define $\gls{o} := \gls{permutation}_{\h}^\top\H$. Then, for all $i,j$,
    \begin{align}
        \sum_{k}\mathrm{O}_{ik}\lambda_k\delta_{kj} &= \sum_{k}\lambda_i\delta_{ik}\mathrm{O}_{kj},\\
        \Leftrightarrow \mathrm{O}_{ij}\lambda_j &= \lambda_i\mathrm{O}_{ij},\\
        \Leftrightarrow \mathrm{O}_{ij}(\lambda_j - \lambda_i) &= 0.
    \end{align}
    Hence, for distinct eigenvalues, all off-diagonal entries of \gls{o} are zero. Within the space of equivalent eigenvalues, \gls{o} is only constrained by the requirement to be orthogonal. 
    
    From $\gls{o} = \gls{permutation}_{\h}^\top\H$ we then know that $\H$ is --- up to permutation --- an orthogonal block-diagonal matrix where the size of each block corresponds to the multiplicity of the values on the diagonal of $\Lam$.

    \end{scpcmd}
\end{proof}
    
\end{scpcmd}

\subsection{Including hard negative sampling} \label{subsec:theory_hard_negatives}

\begin{scpcmd}[
\newscpcommand{\Z}{}{\mathcal{Z}}
\newscpcommand{\d}{}{d}
\newscpcommand{\S}{[1][\d-1]}{\mathbb{S}^{#1}}
\newscpcommand{\D}{}{D}
\newscpcommand{\X}{}{\mathcal{X}}
\newscpcommand{\g}{}{\myvec{g}}
\newscpcommand{\z}{}{\myvec{z}}
\newscpcommand{\pz}{}{\z^{+}}
\newscpcommand{\nz}{[1]}{\z^{-}_{#1}}
\newscpcommand{\x}{}{\myvec{x}}
\newscpcommand{\px}{}{\x^{+}}
\newscpcommand{\nx}{[1]}{\x^{-}_{#1}}
\newscpcommand{\Lam}{}{\myvec{\Lambda}}
\newscpcommand{\pLam}{}{\Lam^{+}}
\newscpcommand{\nLam}{}{\Lam^{-}}
\newscpcommand{\infLam}{}{\myvec{\hat{\Lambda}}}
\newscpcommand{\f}{}{\myvec{f}}
\newscpcommand{\distancesquare}{[4][]}{#1\lVert#2-#3#1\rVert^2_{#4}}
\newscpcommand{\anisoexpsim}{[4][]}{e^{-\distancesquare[#1]{#2}{#3}{#4}}}
\newscpcommand{\M}{}{M}
\newscpcommand{\aniclloss}{}{\gls{cllossgen}}
\newscpcommand{\latentloss}{}{\tilde{\mathcal{L}}}
\newscpcommand{\h}{}{\myvec{h}}
\newscpcommand{\zt}{}{\tilde{\z}}
]

\begin{definition}[Anisotropic DGP with hard negatives (HN)]\label{def:adgp_hn}
    Let $\g:\Z\defeq \S\rightarrow\X \subseteq \rr[\D]$ be an invertible and continuous function. Let $\pLam, \nLam \in \rr[\d\times\d]$ be positive definite diagonal matrices such that $\pLam_{ii} > \nLam_{ii}$ (or equivalently, $\pLam-\nLam$ is still PD).
    
    The following process is called an \emph{anisotropic DGP with hard negatives (HN)}, characterized by the triple $(\g, \pLam, \nLam)$:
    \begin{enumerate}[label=\textit{Step \arabic*)},leftmargin=1.5cm]
        \item $\z \sim Uni(\Z)$
        \item $\pz \sim \dist[\pz|\z](\cdot|\z)$, where $\dist[\pz|\z](\pz|\z) \propto \anisoexpsim{\pz}{\z}{{\pLam}}$
        \item {for all $j \in \{1,\ldots,\M\}, \nz{j} \sim \dist[\nz{}|\z](\cdot|\z)$, where $\dist[\nz{}|\z](\nz{}|\z) \propto \anisoexpsim{\nz{}}{\z}{\nLam}$}
        \item $\x, \px, \{ \nx{j}\} \overset{\g}{\longleftarrow} \z, \pz, \{ \nz{j} \}.$
    \end{enumerate}    
\end{definition}

\begin{repcorollary}{cor:hn_ident}[Identifiability of anisotropic DGPs with HN]
    Let an ADGP with HN (\cref{def:adgp_hn}) be characterized by $(\g,\pLam, \nLam)$ with latent and observed data being $(\z,\pz,\{\nz{j}\})$ and $(\x,\px,\{\nx{j}\})$. Let $\aniclloss$ denote the \ourloss loss (\cref{def:aclp}), evaluated \wrt $(\x,\px,\{\nx{j}\})$.
    
    If a pair $(\f, \infLam)$ (globally) minimizes the $\aniclloss$ loss, then:
    \begin{enumerate}
        \item $\infLam$ is equal to $\pLam - \nLam$ up to a permutation of elements and
        \item $\f\circ\g$ is an orthogonal transformation, or latent $\z$ is identified up to an orthogonal transformation.
    \end{enumerate}
\end{repcorollary}

\begin{proof}
    We are following the steps of the Proof of Thm.~\ref{theo:ident_ADGP}.

    First, we again rewrite the $\aniclloss$ by plugging in $\x = \g(\z)$ and receiving $\aniclloss(\f,\infLam)=\latentloss(\f\circ\g, \infLam)$, with optimization over $\f$ (and $\infLam$). As $\h = \f\circ\g$ attains all possible continuous functions, we may optimize $\latentloss(\h, \infLam)$.
    The hypothesis class becomes, again, $\tilde{\mathcal{U}} = \{ u \,|\, u(\z,\pz)=-\distancesquare{\h(\pz)}{\h(\z)}{\infLam} \text{ for any pair } (\h, \infLam)\}$.
    
    Secondly, according to \cref{theo:bayes_optima_cl}, an arbitrary similarity function $u$ minimizes $\latentloss$ amongst all possible measurable $u$'s (attaining the Bayes-optimum) if and only if for a suitable function $c$:
    \begin{equation}\label{eq:equation_Bayes_optimum_ADGP_HN}
        u(\z,\zt) = c(\z) + \ln\frac{\dist[\pz|\z](\zt|\z)}{\dist[\nz{}|\z](\zt|\z)} \quad\text{holds almost everywhere on $\Z$.}
    \end{equation}
    We are about to plug in all our constraints. The major difference compared to the Proof of Thm.~\ref{theo:ident_ADGP} is that besides the positive conditional distribution being $\dist[\pz|\z](\zt|\z)\propto \anisoexpsim{\zt}{\z}{\pLam}$, the negative conditional also changes to $\dist[\nz{}|\z](\zt|\z)\propto \anisoexpsim{\zt}{\z}{\nLam}$. After plugging in the restricted form of $u$ and merging every constant normalizing term into function $c$, \cref{eq:equation_Bayes_optimum_ADGP_HN} becomes:
    \begin{align}
        -\distancesquare{\h(\zt)}{\h(\z)}{\infLam} &= c(\z) + \ln \frac{\anisoexpsim{\zt}{\z}{\pLam}}{\anisoexpsim{\zt}{\z}{\nLam}}\\
        &= c(\z) - \distancesquare{\zt}{\z}{\pLam} + \distancesquare{\zt}{\z}{\nLam} \\
        &= c(\z) - \scprod{\zt-\z,\pLam(\zt-\z)} + \scprod{\zt-\z,\nLam(\zt-\z)} \\
        &= c(\z) - \scprod{\zt-\z, (\pLam-\nLam)(\zt-\z)} = c(\z) - \distancesquare{\zt}{\z}{\pLam-\nLam}.
    \end{align}
    
    Hence, we get the functional equation with unknowns $\h,\infLam,c$:
    \begin{equation}
        \distancesquare{\h(\zt)}{\h(\z)}{\infLam} = c(\z) + \distancesquare{\zt}{\z}{\pLam - \nLam} \quad\text{holds almost everywhere on $\Z$.}
    \end{equation}

    We conclude our proof by continuing the Proof of Thm.~\ref{theo:ident_ADGP} by substituting $\Lam = \pLam-\nLam$ PD.

\end{proof}

\begin{remark}[Subsuming the uniform marginal in the negative conditional with $\gls{diagtemp}^{-}=0$]
    When using a negative conditional in the form of \cref{eq:weighted_cl_cond}, $\gls{diagtemp}^{-}=0$ accounts for a uniform marginal.
\end{remark}


\subsection{Identifiability of a Loss Ensemble} \label{app:subsec_theory_ensemble}

In this setting the observations \gls{obs} come from $k$ different \glspl{dgp}, with a \textit{shared} generator \gls{dec} and separate concentration parameters $\braces{\gls{diagtemp}^i}_{i=1}^k$. That is, this setting is akin to works investigating identifiability in multiple environments~\citep{hyvarinen_unsupervised_2016,gresele_incomplete_2019,rajendran2023interventional} and reflects the practice in \gls{cl} of optimizing multiple loss components with different data augmentations~\citep{eastwood_self-supervised_2023,xiao_what_2021,zhang_rethinking_2022}.

\end{scpcmd}


\begin{scpcmd}[
\newscpcommand{\Lam}{}{\gls{diagtemp}}
\newscpcommand{\infLam}{}{\gls{diagtempinfer}}
\newscpcommand{\ensLam}{[1]}{\Lam^{\raisemath{-0.1ex}{#1}}}
\newscpcommand{\ensinfLam}{[1]}{\infLam^{\raisemath{-0.1ex}{#1}}}
\newscpcommand{\f}{}{\gls{enc}}
]

\begin{repcorollary}{cor:ensemble_ident}[Identifiability of ensemble \ourloss]

    Assume $k$ \glspl{dgp}, each satisfying \cref{assum:non_isotropic_cl}. Assume that they share \gls{dec}, but have different concentration parameters $\{\ensLam{i}\}$, for $1\leq i \leq k$. Assume a shared encoder \gls{enc} and $k$ learnable $\ensinfLam{i}$ parameters. Then, the tuple $(\gls{enc},\{\ensinfLam{i} \})$ minimizing the ensemble loss $\sum_i^k\gls{cllossgen}(\f, \ensinfLam{i})$ identifies the latents up to a block-orthogonal transformation.

\end{repcorollary}
\begin{proof}
    $\f=\inv{\gls{dec}}$ together with $\ensinfLam{i}=\ensLam{i}$ minimizes the $i^{th}$ loss term. As $\gls{dec}$ is shared, the individual losses can attain their minima simultaneously. Hence, this also holds for any minimum of the ensemble loss. Applying Thm.~\ref{prop:min_ce_maintains_weighted_mse} to either of them concludes the proof.
\end{proof}

\end{scpcmd}

\begin{remark}[Sufficiency condition for identifiability up to permutations for the ensemble \ourloss loss]\label{remark:ensemble_ident_permutations}
    Ensemble \ourloss is identifiable up to permutations if, under the conditions of \cref{cor:ensemble_ident}, there are sufficiently many \glspl{dgp} such that for all distinct latent factors \latenti[k], \latenti[l], there exists a \gls{dgp} with concentration parameter  $\gls{diagtemp}^i$ such that $\gls{diagtemp}^i_{kk} \neq \gls{diagtemp}^i_{ll}$.
\end{remark}
\section{Additional Analysis: Ablations and Parameter Sensitivity} \label{app:analysis}
    In Section~\ref{sec:analysis} of the main manuscript, we identified the use of augmentations as the primary issue in observing a trade-off between accuracy and identifiability on real-world data.
    In this Section, we analyze other reasons why scaling \ourloss to real-world data is challenging.
    In particular, we observe that the interplay between the batch size, the latent dimensionality and the value of the concentration parameter plays a crucial role.
    To demonstrate this, we run controlled experiments on synthetic data to investigate the assumptions and properties of our theory. Since our analysis equally applies to the isotropic and anisotropic cases, we use the isotropic one for our experiments for simplicity but consider the general anisotropic case in our additional theoretical results.

\subsection{Issues when Scaling up \ourloss}

\paragraph{Batch Size Needs to Scale with the Dimensionality.}
    One potential difference between the fully-controlled and the real-world experiments is the latent dimensionality, which will (most likely) be substantially higher for more complex image data. Therefore, we investigate the influence of the latent dimension and batch size on the linear identifiability score (\cref{fig:analysis_1_2}A).
    We observe a significant degradation of performance with a smaller batch size and higher latent dimensionality, corroborating the theoretical predictions of \citet{wang_chaos_2022} and in line with the empirical observation that \gls{cl} performs better with a higher batch size \citep{chen_simple_2020}.
    Even for a modest dimensionality, a rather high batch size is required. Extrapolating our results, feature encoders on the order of what is typically used on CIFAR10 and ImageNet (512 to 2048 dimensions) would require extremely large batch sizes to reach high identifiability scores even if all assumptions match our theory.

      \begin{figure}[t]
        \centering
        \includegraphics[width=0.7\linewidth]{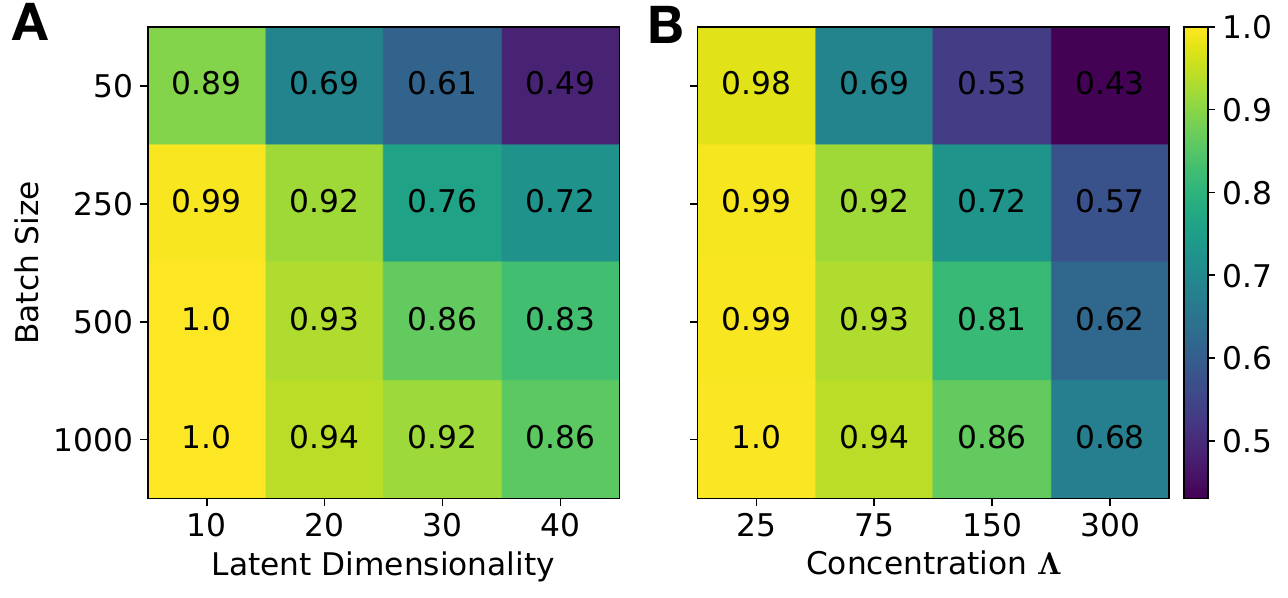}
        \caption{\textbf{Latent Dimensionality, Concentration, and Batch Size Influence Identifiability.} Linear identifiability, quantified by the \gls{r2} score between reconstructed and ground-truth latents, degrades with \textbf{A:} higher latent dimensionality $d$;  and \textbf{B:} larger concentration parameter \gls{diagtemp} in the ground-truth positive conditional. Both detrimental effects can be countered by increasing the batch size. \label{fig:analysis_1_2}}
    \end{figure}
    
\paragraph{Batch Size Needs to Scale with the Concentration \gls{diagtemp}.}
    Another important difference between our controlled and real-world experiments concerns the choice of concentration parameters and its influence on the batch size for identifiability. More concretely, the more concentrated the conditionals are, and the higher the dimensionality, the less training signal CL has to structure the representation~\citep{wang_chaos_2022}. This was denoted by \citet{wang_chaos_2022} as the role of \emph{augmentation overlap}.
    Our anisotropic positive conditional distribution could help explain the role of augmentation overlap from an identifiability perspective via the concentration parameter \gls{diagtemp}, whose effect is illustrated in~\cref{fig:analysis_schematic_lambda}. Larger \gls{diagtemp} means a more concentrated conditional and requires larger batch sizes for sufficient augmentation overlap; otherwise, the \gls{r2} score will degrade (see~\cref{fig:analysis_1_2}B for $\gls{latentdim}=20$). 
    In real-world experiments, the positive pairs distribution is determined by augmentations. 
    Likely, the shape or the class of an object are minimally affected by the augmentations, which corresponds to those latent dimensions having a large concentration parameter, thereby again increasing the necessary batch size to reach high identifiability scores.


\paragraph{Augmentations can Violate our Conditional Assumption.}
    Next, we consider the role of different positive pair distributions.
    In \cref{sec:ident_theory}, we assume a normalized Gaussian distribution over a hypersphere. However, it is unclear whether practical augmentation strategies follow such a conditional.
    We test our assumption's validity by training a \gls{vae} on augmented images from MNIST with only crops and horizontal flips as augmentations since MNIST images are grayscale. Training an encoder on MNIST using these augmentations leads to a KNN accuracy of 96.7\%, validating our augmentation choice. We visualize an unaugmented image and VAE reconstructions of its augmented versions in~\cref{fig:analysis_cond_diff_marginal}A. To show the conditional distribution, we project augmented images onto the learned \gls{vae} latent space and demonstrate in~\cref{fig:analysis_cond_diff_marginal}B
    that the conditional distribution does not follow a normalized Gaussian distribution for most dimensions and can even be bimodal; we show more examples in \cref{fig:analysis_cond}. This holds even though the latent space of the VAE across all augmented data samples is Gaussian by construction. Our theory does not cover this case and so this mismatch presents a fruitful direction for future research on closing the gap between theory and practice.

       \begin{figure}[t]
        \centering
        \includegraphics[width=0.5\linewidth]{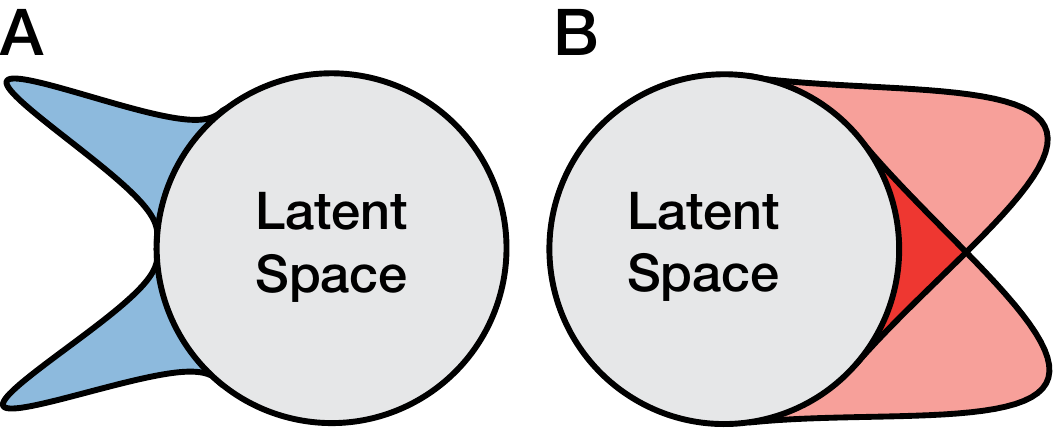}
        \caption{
            \textbf{Concentration of Positive Pairs Influences Augmentation Overlap}.
            We figuratively visualize conditional distributions with large (\textbf{A}) and small (\textbf{B}) concentration parameters \gls{diagtemp}. For large concentration values, samples from the conditional distribution of two anchor points do not overlap, signifying missing augmentation overlap. This is not the case for small \gls{diagtemp} values.
        }
        \label{fig:analysis_schematic_lambda}
  \end{figure}

    \begin{figure*}[t]
        \centering
            \includegraphics[width=\linewidth]{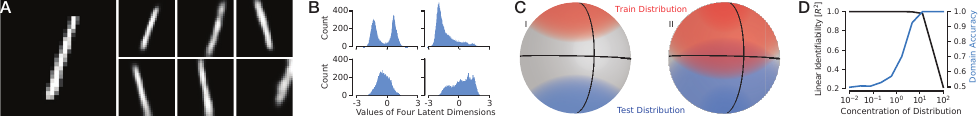}
            \caption{\textbf{Data Augmentation Can Violate Our Data Distribution Assumptions.}  We train a VAE on augmented MNIST images. \textbf{A.} We show an MNIST sample and six augmented versions generated by the VAE. \textbf{B.} We project augmented versions of an MNIST sample into the VAE's latent space and visualize the distribution of four dimensions. This demonstrates a violated conditional assumption of a normalized Gaussian. \textbf{C.} We visualize violations of the uniformity assumption for the marginal distribution during training and testing. \textbf{D.} We observe degrading identifiability scores when the train and test domains can be distinguished. Since an almost perfect domain classifier can be trained on real-world datasets, we conclude that the uniformity assumption is likely violated in practice.
            \label{fig:analysis_cond_diff_marginal}}
    \end{figure*}

    \begin{figure}[t]
        \centering
            \centering
            \includegraphics[width=1.0\linewidth]{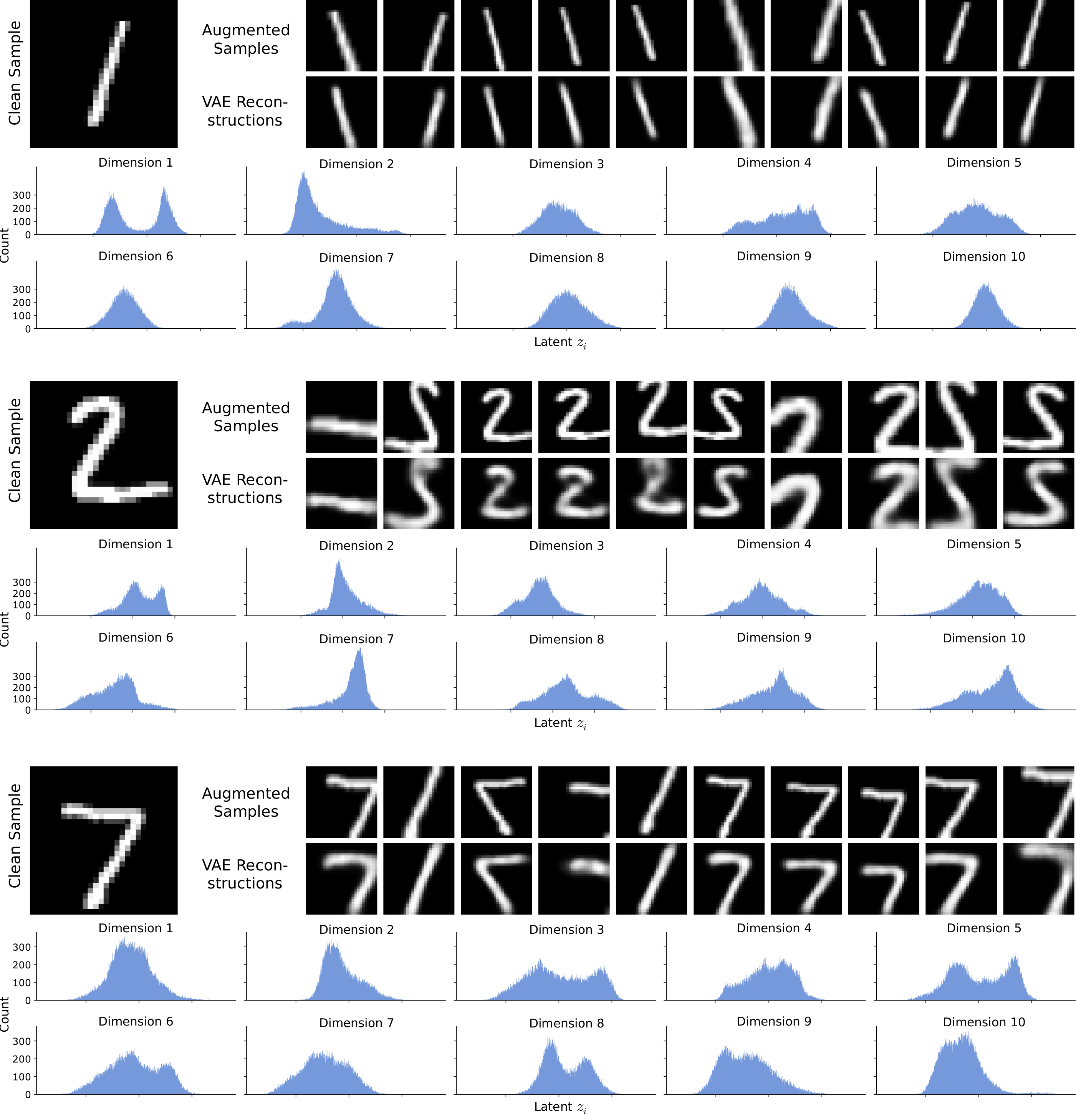}
            \caption{\textbf{When Using Augmentations to Generate Views for \gls{cl}, the Conditional Generally no Longer Follows a Gaussian Distribution.} We show augmented images as well as their reconstructions using a VAE trained on augmented images. We project the augmented images into the latent space of the trained VAE and observe that their distribution does not resemble a Gaussian distribution and can even be bimodal. \label{fig:analysis_cond}}
    \end{figure}

\paragraph{Augmentations can Violate our Uniformity Assumption.}
    Our theory assumes the same uniform marginal \clmarginal for both training and testing. However, this assumption's validity is unclear in practice: During training, even anchor points are produced by data augmentation, creating a potential mismatch to the test domain which lacks these augmentations. 
    We test the influence of violations of this assumption by interpolating the marginal between a uniform and a \gls{vmf} distribution with varying concentration. We center the \gls{vmf} at the north (south) pole during training (testing) (\cref{fig:analysis_cond_diff_marginal}C). 
    We train a binary classifier to distinguish samples from the training and test distributions. For different concentrations, we compute the domain classification accuracy and the identifiability score (\cref{fig:analysis_cond_diff_marginal}D), showing a stark anticorrelation.
    Our domain classification results with regular SimCLR augmentations (MNIST: $99.1$\%, CIFAR10: $99.2$\%, and ImageNet $98.7$\%) demonstrate that binary classifiers can almost perfectly distinguish augmented and non-augmented samples on real-world data. Connecting this to observed anticorrelation with identifiability, we conclude that the mismatch between distributions hurts the quality of representations.

    \paragraph{The Conditional's Effect on Learning Dynamics.}
        Finally, the generative process of real-world image data is most likely substantially richer compared to the previously studied synthetic DGPs, resulting in a generally harder learning problem requiring longer training time. In particular, we hypothesize that the conditional distribution's concentration partly influences the learning speed; this aligns with previous observations of flat loss landscapes for \gls{nce} \citep{liu_analyzing_2021}.
        Our anisotropic conditional of positive pairs provides a possible explanation: We hypothesize that similar to \citet{saxe_exact_2014}, the factors controlling the variances (in our case, the concentration \gls{diagtemp}) affect the loss landscape, making some latents \textit{slow} to learn.
        Thus, it is conceivable that more style-like (ground-truth) latents are learned slower (thus, harder) compared to more content-like ones (see, \eg, in ~\cref{fig:toymodel}B). 

    We demonstrate the influence of the concentration parameter on the learning dynamics as follows.
    We define two DGPs using the synthetic data from Section~\ref{subsec:lvm}.
    We assume a uniform marginal and an isotropic normalized Gaussian conditional with two different \gls{diagtemp} parameters: \gls{diagtemp} = 25 and \gls{diagtemp} = 75. A higher concentration parameter corresponds to a more narrow conditional distribution and thus, lower variation between the anchor and its conditional sample in the latent space.
    We analyze how \gls{diagtemp} affects learning dynamics  (\cref{fig:analysis_bsz_iter}) and observe that learning the latents follows a step-like behavior, corroborating~\citet{simon_stepwise_2023}, where latents with larger \gls{diagtemp} are learned slower (\cref{fig:analysis_bsz_iter} A \& D). Alternatively, reaching a target \gls{r2} score requires more gradient steps unless the batch size is increased (\cref{fig:analysis_bsz_iter} B \& E). The cosine similarities' distribution for negative and positive pairs clearly shows that large \gls{diagtemp} leads to insufficient augmentation overlap: Since the histograms are disjoint for large \gls{diagtemp}, the anchor's nearest neighbor is always a positive sample (\cref{fig:analysis_bsz_iter} C \& F). This makes distinguishing positive from negative samples trivial and allows for a short-cut solution. However, augmentation overlap cannot fully explain \gls{cl} in practice:  \citet{saunshi_understanding_2022} show that practical augmentations often provide no significant overlap; nonetheless, \gls{cl} still succeeds, potentially due to (architectural) inductive biases.
    \begin{figure}[t]
        \centering
            \includegraphics[width=1.0\linewidth]{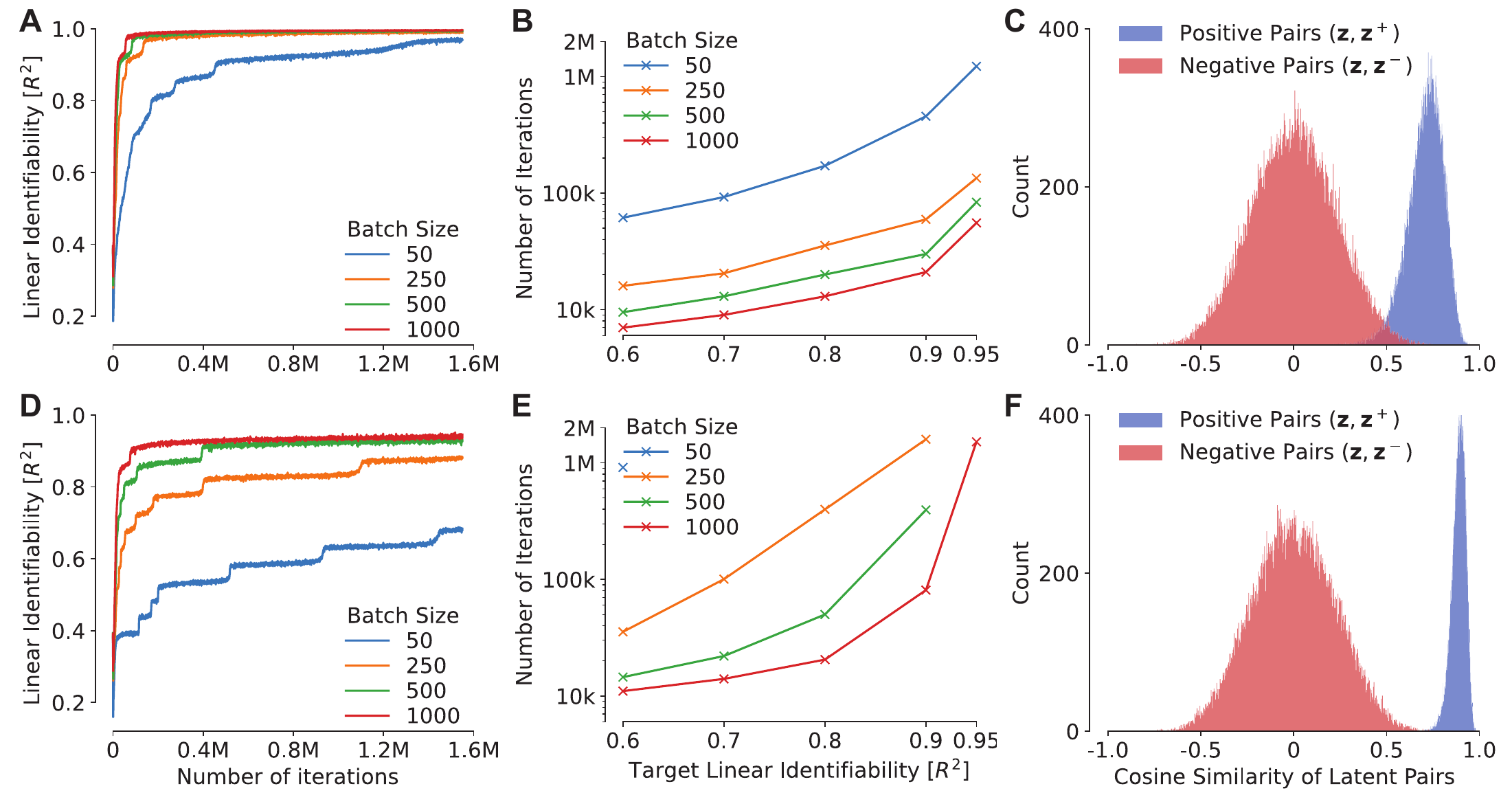}
            \caption{\textbf{Analysis of Learning Dynamics for a Small ($\gls{diagtemp}=25$, \textbf{A}-\textbf{C}) and a Large ($\gls{diagtemp}=75$, \textbf{D}-\textbf{F}) Concentration Parameter.} \textbf{A \& D}: Linear identifiability measured by the $R^2$ scores show step-like behavior and generally need more training iterations to reach a certain $R^2$ score for higher $\gls{diagtemp}$. \textbf{B \& E}: The number of training iterations needed for a certain target $R^2$ score grows exponentially. \textbf{C \& F}: We explain degrading $R^2$ scores for higher concentration parameters $\gls{diagtemp}$ with missing overlap in the latent positive and negative distributions. \label{fig:analysis_bsz_iter}}
    \end{figure}

\subsection{Dependence of Identifiability on Batch Size, Training Duration, and Dimensionality} \label{app:depend_ident_bsz_dim_iters}

    We now further analyze the dependence of the batch size, training duration, and dimensionality of the latent space on the identifiability scores. To this end, we train models with different batch sizes on data with various latent dimensionality and investigate their identifiability scores. We use $3$ random seeds per configuration and train all models for $500$k iterations using the same hyperparameters as described above. We consider the usual conditional distribution of positive pairs with a concentration value $\gls{diagtemp}=75$.
    
    \begin{figure}[t]
        \centering
        \includegraphics[width=0.7\linewidth]{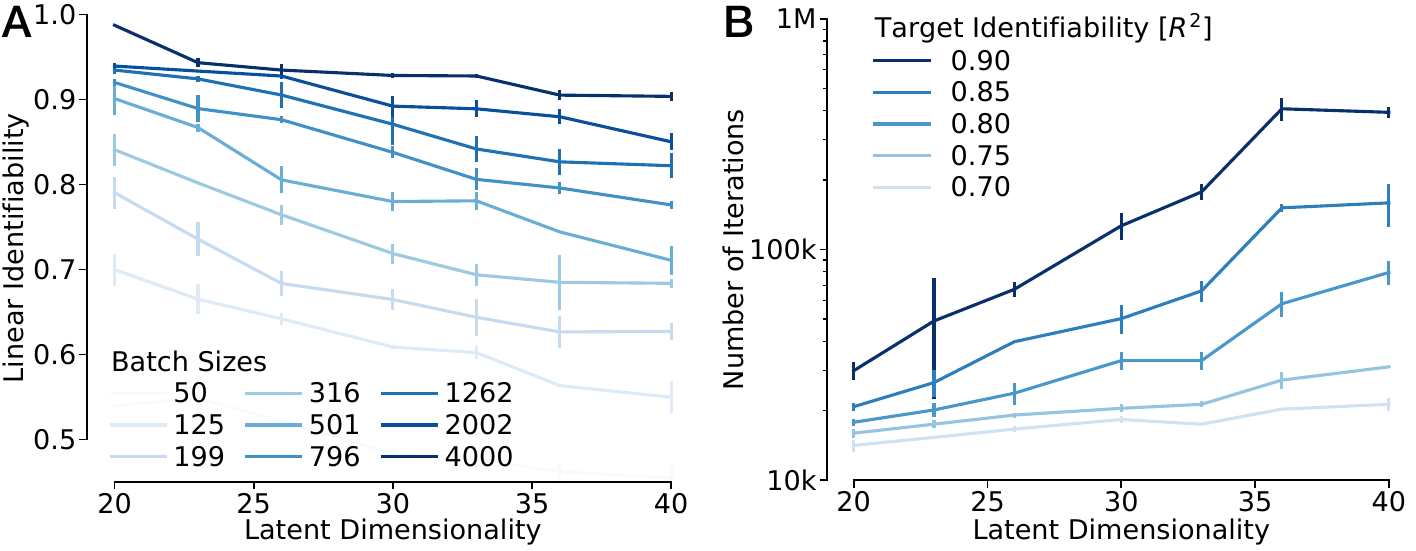}
        \caption{
            \textbf{Higher Dimensional Latent Spaces Require Longer Training with Larger Batch Sizes.}
            \textbf{A}: The linear identifiability achieved after training for $500$k steps decreases with an increasing latent dimensionality but increases with larger batch sizes. \textbf{B}: We display the required number of iterations to achieve various target linear identifiability scores ($0.70$ to $0.90$) for a fixed batch size of 4000. Note that the number of iterations required grows exponentially with the latent dimensionality. 
        }
        \label{fig:identifiability_batch_size_iterations} 
    \end{figure}

    \Cref{fig:identifiability_batch_size_iterations} clearly shows that for higher dimensional latent spaces, one needs to train (exponentially) longer and use an increased batch size. Note that even for the fairly low dimensional settings investigated here, large batch sizes and long training runs were necessary to achieve high identifiability scores. Comparing this to the dimensionalities usually assumed/used in practice, e.g., $512$ dimensions when training on images, we argue that the usual batch sizes and training durations are too small to achieve high identifiability scores. This observation would explain the mixed performance observed earlier when testing our loss on real-world image data despite its identifiability guarantees.

\subsection{Hard Negative Mining and Ensemble AnInfoNCE}
\label{app:mitigation}
\paragraph{Hard Negative Mining.}
    In the previous section, we observed that high concentration parameters require higher batch sizes for successful learning.
    \Gls{hn} mining can help counter highly concentrated conditionals by increasing the sample density close to the anchors.
    First, we demonstrate in our synthetic setting for $d=20$ that with \gls{hn} mining, the optimum for the loss in terms of $\hat{\gls{diagtemp}}$ is indeed at the predicted value of $\gls{diagtemp} = \gls{diagtemp}^+\!\!-\!\gls{diagtemp}^-\!$ (\cref{fig:hn_opt_gamma}A).
    We also finetune an encoder using \ourloss by adding three hard negatives for each sample. 
    We observe better linear identifiability scores compared to continued regular training, supporting our \cref{cor:hn_ident}, (\cref{fig:hn_opt_gamma}B). 

    \begin{figure}[t]
        \centering
            \includegraphics[width=0.7\linewidth]{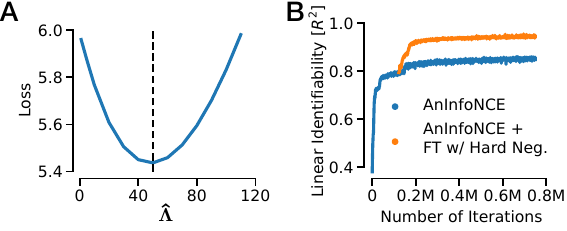}
            \caption{\textbf{Hard Negative Mining Improves Identifiability.} \textbf{A:} For $\gls{diagtemp}^+=150$ and $\gls{diagtemp}^-=100$, we find a loss optimum at $\gls{diagtemp} = \gls{diagtemp}^+-\gls{diagtemp}^-=50$ as predicted by \cref{cor:hn_ident}; \textbf{B:} Finetuning (FT) an encoder on a concatenation of hard negative samples and regular negative samples improves the identifiability score. 
            \label{fig:hn_opt_gamma}}
    \end{figure}
  
\paragraph{Training With Ensemble \ourloss.}
    To test the benefit of ensembling \ourloss for training, we define two \glspl{dgp} with $\gls{latentdim}=20$.
    The \glspl{dgp} share a uniform marginal distribution and differ in their respective conditional distributions.
    To model latents that cannot be changed by either DGP, such as object class or shape, we set a high concentration parameter of 400 for three latents shared between both DGPs.
    For the remaining seventeen latents, we choose two $\gls{diagtemp}$ values of 15 and 250, modeling strongly and weakly varying latents and alternate them between the DGPs:
    \begin{align*}
        \gls{diagtemp}_1 &= \diag{400, 400, 400, 15, \ldots, 15, 250, \ldots, 250} \\
        \gls{diagtemp}_2 &= \diag{400, 400, 400, 250, \ldots, 250, 15, \ldots, 15}.
    \end{align*}

    We sample data with both DGPs and use two losses with two learnable concentration parameters $\gls{diagtempinfer}_1,\gls{diagtempinfer}_2$ to model both DGPs. We train the encoder on the sum of both losses.
    We see improved linear identifiability scores when training the encoder on the loss ensemble compared to InfoNCE or \ourloss (see also~\cref{fig:ensemble_lambdas}A): 
    \begin{center}
    \small
        \setlength{\tabcolsep}{4pt}
        \begin{tabular}{l c c c}
        Loss &  InfoNCE   & \ourloss  & Ens. \ourloss \\
             \midrule
      Lin. ident. $[R^2]$  & 0.48 & 0.93 & 0.98 \\ 
        \end{tabular}
    \end{center} 
    %

    Further, we observe a specialization in the learned $\hat{\Lambda}$ values (\cref{fig:ensemble_lambdas}B+C), where same colors indicate the same $\hat{\Lambda}$ dimension): We observe an anti-correlation between $\hat{\Lambda}_1$ and $\hat{\Lambda}_2$.
    Optimizing an ensemble loss makes it more likely that for each latent dimension, there exists an augmentation making the dimensions more content-like with a higher \gls{diagtemp}.
    
    \begin{figure}[t]
        \centering
            \includegraphics[width=0.9\linewidth]{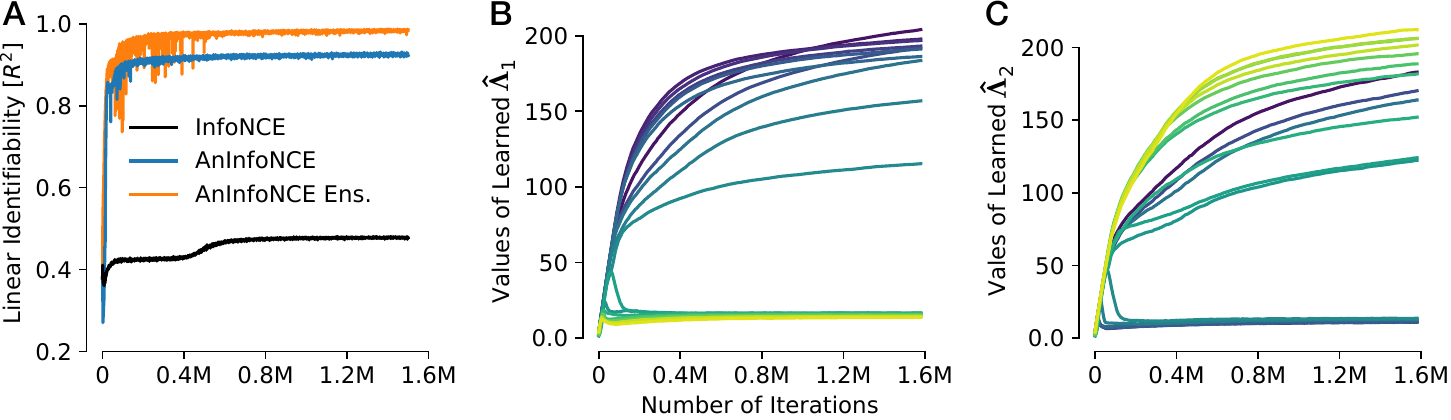}
            \caption{\textbf{Ensemble \ourloss Improves Upon InfoNCE and \ourloss:}
            We observe higher linear identifiability scores when training with Ensemble \ourloss.
            When examining the learned $\hat{\Lambda}$ values, we see a specialization in the learned $\hat{\Lambda}$ values: The color scheme is chosen such that the same color indicates the same dimension, and we find that the learned $\hat{\Lambda}$-values for the two DGPs are anti-correlated with respect to each other. \label{fig:ensemble_lambdas}}
    \end{figure}

\subsection{Comparison to TriCL \citep{zhang2023tricontrastive}}
\label{sec:tricl}

We first discuss the conceptual differences, theoretical results and assumptions  betwen TriCL and \ourloss. TriCL extends and theoretically analyzes the spectral contrastive loss formulation (their extension to InfoNCE does not have a proof). Our paper extends InfoNCE and provides an identifiability result in that case. Note that the extensions are also different: We weigh the differences of latent coordinates instead of the latents themselves.
The TriCL paper provides a result on encoder output equivalence (similar to \citet{roeder_linear_2020}, meaning that all encoders optimizing the TriCL loss are linearly related, or, more precisely $f_1(x)=Af_2(x)$. However, the encoder’s output is not related to ground-truth latents in any way. Our paper provides a result on the equivalence of the output with the ground-truth latents, meaning that $f \circ g (z) = Oz$. On one hand, our type of result subsumes the TriCL type of result. Moreover, our result is well-suited for the qualitative assessment of inferred latent features, which is not possible with the latter.
The theoretical result from TriCL requires that an additional decorrelation regularizer completely vanishes, whereas our method does not have any such term.

We also show an experimental comparison between TriCL and \ourloss below. First, we noticed a small discrepancy between the paper’s definition of TriCL and its experimental implementation in how the loss is computed. Based on the used implementation, the scale parameter is not applied to all terms in the sum but only to half of them. We are not sure how how this discrepancy affects the experiments.

We compare the evolution of linear identifiability scores when training with TriCL or \ourloss in our synthetic experiment in Fig.~\ref{fig:tricl}. We set the latent dimensionality to 20 and test two cases for the positive conditional distribution. (1) Uniform $\Lambda$ set to 100 (top row); (2) Dual partitioned $\Lambda$, such that half of the dimensions have a low $\Lambda$ of 15 (”style”-like dimensions) and the other half high $\Lambda$ of 250 (”content”-like dimensions) (bottom row). In both cases, we observe that \ourloss behaves much more stably compared to TriCL and reaches higher linear identifiability scores, especially when there is an anisotropic conditional.

   \begin{figure}[t]
        \centering
            \includegraphics[width=0.8\linewidth]{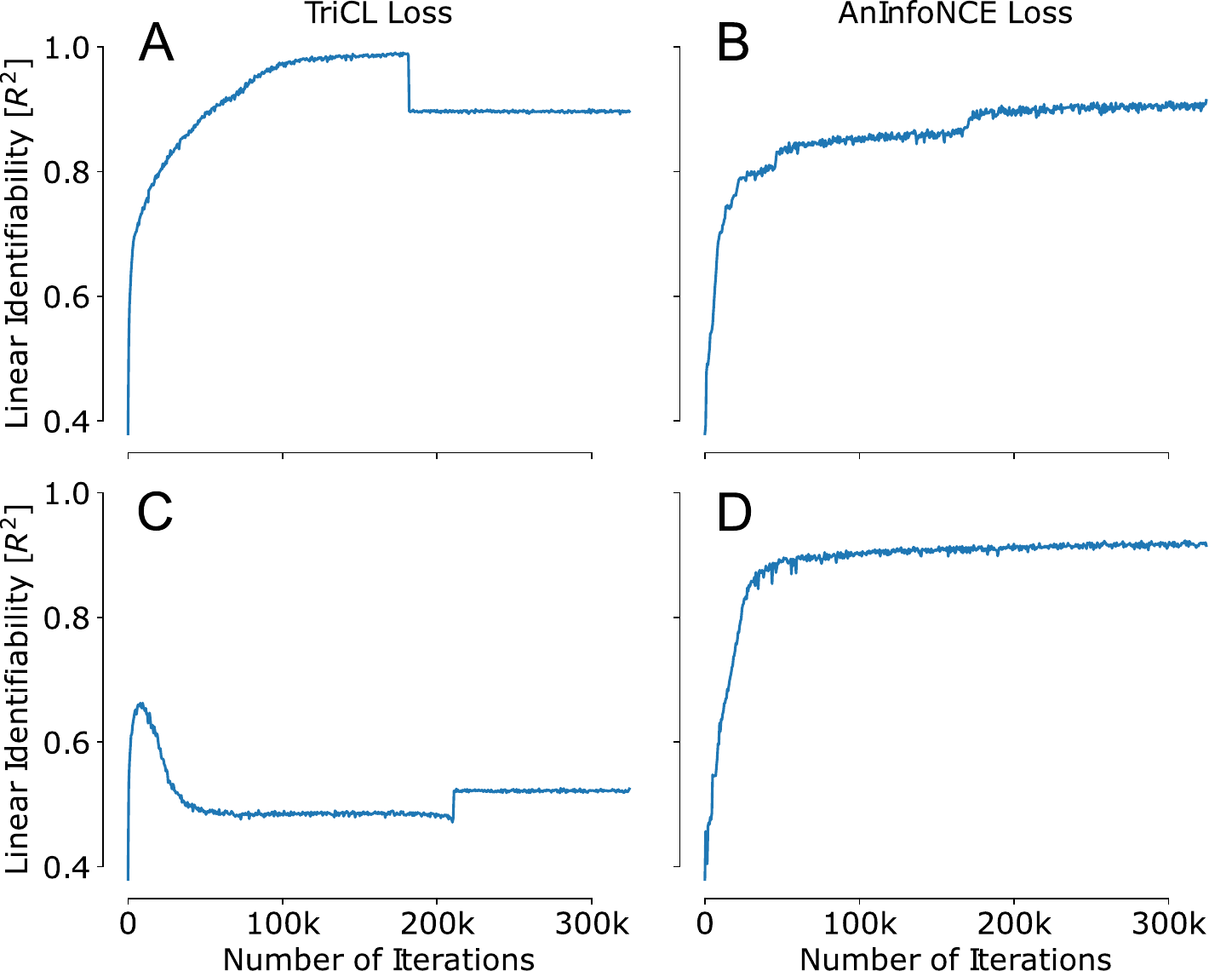}
            \caption{\textbf{Synthetic Experiments for a Latent Dimensionality of d=20 When Training With TriCL or AnInfoNCE.} Top row: Uniform $\Lambda$ of 100. Bottom row: Dual partitioned $\Lambda$, such that half of the dimensions are varied according to a low $\Lambda$ of 15 and the other half is varied with a high $\Lambda$ of 250. In both cases, we observe that AnInfoNCE behaves much more stably compared to TriCL and reaches higher linear identifiability scores, especially when there is an anisotropic positive conditional. \label{fig:tricl}}
    \end{figure}

\subsection{Training Details for VAE Training in MNIST} 
\label{app:mnist}

    \begin{figure}[t]
        \centering
            \includegraphics[width=0.8\linewidth]{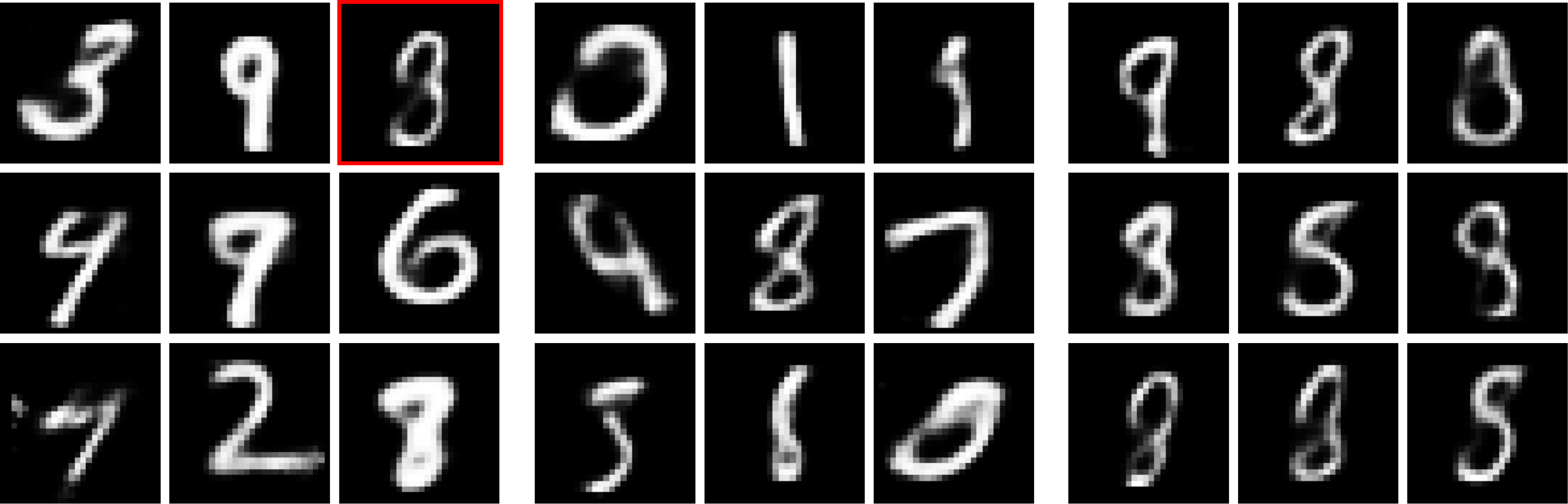}
            \caption{\textbf{Samples Generated by the MNIST-VAE}. \textbf{Left}: Samples from the marginal distribution, used as anchor samples for the contrastive objective. \textbf{Middle \& Right}: For the highlighted sample, we show potential positive samples from the conditional distribution for two different values of $\gls{diagtemp}$. A value of $5$ leads to too much variation (middle) and leads to obvious class changes in the conditional samples, whereas a value of $50$ leads to reasonable variation. We note that training an encoder on images which were sampled according to a DGP with either ground-truth $\Lambda$-value leads to perfect linear disentanglement scores. In other words, a low concentration parameter of $\Lambda=5$ which changes the class of the conditional view does not pose an issue.  \label{fig:mnist_samples}}
    \end{figure}

    The VAE consists of three fully connected layers with Leaky ReLU activations and is trained on MNIST for 50 epochs. The latent dimensionality \gls{latentdim} of the VAE is set to 10.
    We use the trained VAE to generate training images in the following way: We first sample a latent vector $z$ from the marginal distribution and then sample a vector conditioned on $z$. Since both vectors are normalized to the sphere surface, we multiply them with $\sqrt{d}$ to project them to a Gaussian vector. Finally, we use the VAE decoder to generate images which we train the encoder with the \ourloss loss on.
    
    To provide an intuition for the quality of samples generated by the VAE, we show images generated from random samples from the latent space in Fig.~\ref{fig:mnist_samples} (left). For the highlighted sample in the top-right corner, we also show samples from the conditional distribution for two different $\Lambda$ values: $\Lambda=5$ (middle) and $\Lambda=50$ (right).

\subsection{Training details and additional experiments for C-3DIdent}
\label{app:cident}
The Causal 3D-Ident dataset \citep{von_kugelgen_self-supervised_2021} is an extension of the 3D-Ident dataset \citep{zimmermann_contrastive_2021} to six more classes. The training split has 252000 images and the test split has 25200 images.
We use the original code base from \citep{von_kugelgen_self-supervised_2021} with the original training hyperparameters.
The latent dimensionality is ten and the latent space of C-3DIdent is a hypercube. Following \citet{von_kugelgen_self-supervised_2021}, we use cosine similarity as our similarity measure which restricts the inferred latents to a hypersphere. To enable our encoder learning all hypercube latents, we set the latent dimensionality to eleven.
We train all models for 100K gradient steps.

\section{Compute requirements}
\label{app:compute}

\paragraph{Synthetic experiments.}
For our synthetic experiments, as described in Sections~\ref{subsec:lvm}, \ref{sec:vae_mnist} as well as in Appendix~\ref{app:analysis}, we used a single NVIDIA GeForce RTX 2080 Ti with 11GB of RAM for each experiment. We used 8 CPUs and 8 workers per experiment. The convergence time for each experiment varied, e.g. higher batch size or a lower concentration parameter $\hat{\gls{diagtemp}}$ generally imply faster convergence: We show the relationship between the batch size and the concentration parameter in  Fig.~\ref{fig:analysis_bsz_iter}. We set the run time for synthetic experiments to ten hours to make sure that all runs converge; the main experiments shown in Figs.~\ref{fig:toymodel} and \ref{fig:mnist_scores_acc} train within one hour to convergence.

\paragraph{Real-world experiments.}
On CIFAR10, we trained the ResNet18 model on a single NVIDIA A100 GPU with 40GB RAM and 8 CPUs (with 8 workers). The training took about seven hours per run. The linear readout evaluation as well as the binary augmentations readout took about 30 minutes on the same compute setup.

On ImageNet, we trained a ResNet50 on 8 NVIDIA A100 GPUs with 40GB RAM and 64 CPUs. We set the number of workers to 16. The training took about 15 hours on this setup. The linear readout evaluation as well as the binary augmentations readout took about 3 hours using one GPU.

On C-3DIdent, we trained a ResNet18 model on a single NVIDIA A100 GPU with 40GB RAM and 32 CPUs (with 32 workers). The training took about fourteen hours per run. 

\paragraph{Compute costs per experiment, main paper.}
While we used different GPUs between our synthetic and real-world experiments, we will use a metric of generic ``GPU-hours'' for both cases. As described below, both GPU types require about the same amount of power.
\begin{enumerate}
    \item Synthetic data, Fig.~\ref{fig:toymodel}: Experiments showing the influence of linear identifiability scores when varying \gls{diagtemp}. We show results for 6 values of \gls{diagtemp} for InfoNCE and AnInfoNCE. The GPU run time for this experiment is $6\cdot 2 \cdot 10h=120h$.
    \item Synthetic data, Fig.~\ref{fig:mnist_scores_acc}: Experiments showing how AnInfoNCE behaves in a slightly more difficult experimental setting on MNIST. We show results for 3 values of \gls{diagtemp} for InfoNCE and AnInfoNCE in Fig~\ref{fig:mnist_scores_acc} A+B, and for 2 values of \gls{diagtemp} when training with AnInfoNCE in Fig~\ref{fig:mnist_scores_acc} C+D. The GPU run time for this experiment is $3\cdot 2 \cdot 10h + 2\cdot 10h=80h$.
    \item Real-world data, Table~\ref{tab:augmentations_readout}: We show a trade-off between augmentations readout and downstream accuracy on CIFAR10 and ImageNet. We show results for training with InfoNCE and AnInfoNCE both on CIFAR10 and ImageNet. On CIFAR10, we additionally show results for training with AnInfoNCE without $\ell_2$ normalization of the output before calculating the loss. We also conduct the linear readout evaluation on both the backbone and the post-projection features. The compute break-down is as follows:
    \begin{enumerate}
        \item CIFAR10:  Training: $3\times7h=21h$, Linear readout: $3\cdot6\cdot0.5h\cdot 2=18h$. Total: $39h$
        \item ImageNet: Training: $2\times15h\cdot 8$~GPUs$=240h$, Linear readout: $2\cdot3h\cdot6\cdot2=72h$. Total: $312h$
    \end{enumerate}
    The total compute for the real-world experiments is then 351 GPU hours.
    \item MNIST, inset Table in Sec.\ref{sec:analysis}; showing the accuracy-identifiability trade-off. The run time for this experiment has been: $2\cdot 10h$=20$h$.
    \item C-3DIdent, \cref{tab:3dident}; we compare InfoNCE with \ourloss on C-3DIdent and investigate the accuracy and identifiability trade-off. We tested 2 view generation strategies (via data augmentation and via ground-truth sampling) and ran 3 random seeds for two different loss functions. The total run time for this experiment has been: $2\cdot 3\cdot 2\cdot 14h=168h.$
\end{enumerate}

\paragraph{Compute costs per experiment, Appendix.}
\begin{enumerate}
    \item Synthetic data, Fig.~\ref{fig:analysis_1_2}. We show ablation experiments for varying the batch size, the latent dimensionality and the concentration. We show results for 32 different parameter combinations. The GPU run time for this experiment is $32 \cdot 10h=320h$.
    \item Synthetic data, Fig.~\ref{fig:analysis_bsz_iter}. We analyze the interplay between higher dimensional latent spaces and batch size in more detail. We ran 3 random seeds in this experiment, analyzed 9 different batch sizes across 5 different latent dimensionalities. The total run time for this experiment is $3\cdot5\cdot9\cdot10h=1350h$.
    \item Synthetic data, Fig.~\ref{fig:identifiability_batch_size_iterations}. We utilize the previous runs from computing results for Fig.~\ref{fig:analysis_1_2}; this Figure does not incur additional compute costs.
    \item Hard-negative mining, Fig.~\ref{fig:hn_opt_gamma}. We show that hard-negative mining can improve identifiability. We finetune one model trained with regular InfoNCE on AnInfoNCE with hard negative examples. The run time for this experiment is $3h$.
    \item Synthetic data, Fig.~\ref{fig:ensemble_lambdas}. We show that ensembling can improve upon InfoNCE and AnInfoNCE. The run time for this experiment is $3\cdot10h=30h$.
    \item Synthetic data, Fig.~\ref{fig:tricl}. We show that \ourloss is more stable and reaches higher linear identifiability scores compared to the TriCL loss. The run times for this experiment is $2\cdot10h=20h$
\end{enumerate}

\paragraph{Total compute estimation.}
Given the previous paragraphs, the training time estimation is $2774$~GPU hours. The maximum power consumption for the 2080Ti GPUs is 250W \citep{2080ti_power} and 300W for the A100 GPUs \citep{a100_power}. We assume an upper bound on power consumption of 300W per GPU. Then, the energy spent by the GPUs for the experiments shown in this paper is 832.2kWh. We estimate an additional factor of 20 in terms of energy spent during the development stage of this project. Therefore, the total estimated project cost in terms of energy is about 16.6MWh. This corresponds to 3.4t CO$_2$ emissions \citep{co2_calc} which is roughly equivalent to six single-person flights from London to New York \citep{flight_co2}.

\section{Software stack}
\label{app:software}
    We use different open source software packages for our experiments, most notably Docker \citep{10.5555/2600239.2600241}, Singularity \citep{singularity},
    scipy and numpy \citep{2020SciPy-NMeth},
    PyTorch \citep{paszke2017automatic},
    and
    torchvision \citep{10.1145/1873951.1874254}.



\section{Additional related works}

\subsection{An \gls{lvm} view of SSL}

\citet{arora_theoretical_2019} developed generalization bounds for downstream classification in contrastive learning, with one or more negative samples, assuming the presence of latent class variables and (for some results considering intra-class concentration) sub-Gaussian conditionals. Results also show that more negative samples can hurt performance when the negative samples have the same class (c.f. their Sec. 4.2).
\citet{zimmermann_contrastive_2021} connected SimCLR to nonlinear ICA and proves that the SimCLR loss is equivalent to the cross entropy between ground-truth and approximate conditionals. Based on this result, identifiability results are provided for latent spaces structured as hyperspheres or convex bodies - our theory extends the results of \citet{zimmermann_contrastive_2021}.
\citet{von_kugelgen_self-supervised_2021} introduced the content-style terminology and proved that content variables can be recovered with non-contrastive SSL under certain assumptions.
\citet{daunhawer_identifiability_2023} extended these results to the multi-modal case where some latent variables are modality-specific.
\citet{lyu_latent_2021} extended \citet{von_kugelgen_self-supervised_2021} to get guarantees on the identifiability of style variables as well, though they require invertible encoders, whereas \citet{von_kugelgen_self-supervised_2021} have results for the non-invertible case. However, \citet{lyu_latent_2021} admit different encoders across views.
\citet{lachapelle_synergies_2022} investigated the role of sparsity for identifiability in the multi-task scenario, i.e., when the representation is used to solve multiple downstream tasks. This setting is similar to that in \citet{dubois_improving_2022}, since the authors used downstream tasks that rely on all latent factors, i.e., for good performance all of them need to be identified. Inducing sparsity in the linear classification head yields an optimal and disentangled representation with identifiability guarantees up to permutation and scaling.
\citet{cui_aggnce_nodate}, building on \citet{zimmermann_contrastive_2021,von_kugelgen_self-supervised_2021}, introduced the AggNCE loss function and proves that with multiple views, the InfoNCE loss is identifiable up to an invertible matrix.
\citet{lee_predicting_2021} quantified how specific pretext tasks can guarantee learning a good representation, i.e., one which can solve the downstream task by only using a linear layer on top. The authors relied on approximate conditional independence between the input and the pretext target (conditioned on the downstream task label and the latents).
\citet{eastwood_self-supervised_2023} proposed a new SSL objective inspired from invariance-based methods that allow for the disentanglement of style variables alongside content. This paper can also be seen as a more rigorous treatment of \citet{xiao_what_2021}. However, their identifiability results differ from ours as they mostly rely on multiple sets of augmentations, and a different model and loss pair.
\citet{yao_multi-view_2023} unified and extended the results from \citet{von_kugelgen_self-supervised_2021,daunhawer_identifiability_2023} to the multi-view case with partial observability, showing that any shared information across arbitrary subsets of views and modalities can be learned up to a smooth bijection using contrastive learning.

\subsection{The role of data augmentations}

    The (empirical) study of different augmentation strategies relates to our work as the underlying data generating process (DGP) is affected by them and our work studies SSL from the DGP’s perspective. Within this category, we distinguish two subcategories:

\subsubsection{Augmentation strategies with a single latent space}

\citet{wu_mutual_2020} empirically evaluated collision-free augmentations (i.e., when different data points cannot have the same view---e.g., changing brightness is collision-free, but cropping can introduce collisions). They conclude that collision-free augmentations are not useful for the downstream task of classification on ImageNet (c.f. their Figs. 1 and 5) and CIFAR10 (their Fig. 5)---whereas the ones with collision can improve performance.
\citet{mansfield_random_2023} proposed a more flexible class of data augmentations based on Gaussian random fields and showed that it improves performance on ImageNet and iNaturalist, but their method is sensitive to hyperparameters (strong transformations can degrade the image).
\citet{bendidi_no_2023} extensively studied the effect of data augmentations in natural (biological) images w.r.t. downstream (classification) performance and clustering structure. They emphasize that the choice of data augmentation strategies is a form of weak supervision, and its design requires domain expertise to achieve better performance on a given task, since different augmentations can affect the classes differently.
\citet{wagner_importance_2022} introduced the GroupAugment augmentation strategy that searches over a very general space of both augmentations and sampling strategies. The authors also empirically study the role of data augmentation hyperparameters, highlighting that it dramatically impacts performance in SSL.
\citet{ericsson_why_2022} studied learning invariances in contrastive learning and their transfer from synthetic to real-world scenarios. The authors consider two augmentation categories: spatial augmentations change the objects' position, whereas appearance mostly affects color and texture. The authors empirically demonstrate that using augmentations from one of the above groups leads to increased invariance w.r.t. the same augmentation types---these only transfer to some extent from synthetic settings to the real world. However, there is a trade-off in all studied models, i.e., there is no reliable means to learn invariances to both augmentation types. Since different downstream tasks are shown to rely on different invariances, the authors advocate for a fusion of multiple representations, i.e., leveraging multiple representations that were trained by using different types of augmentations.

\subsubsection{Augmentation strategies with combined latent spaces}

\citet{xiao_what_2021} introduced LooC (Leave-one-out Contrastive Learning) using multiple embeddings, where each embedding is trained with a different (sub)set of augmentations. The authors argued and empirically demonstrated that by inducing different invariances in the different embeddings, the combined latent space outperforms all considered baselines on multiple downstream tasks. This setting corresponds to our ensemble approach, with the only difference that we share the encoder, thus, we only have a single latent space.
\citet{zhang_rethinking_2022} proposed a similar strategy to LooC with an add-one augmentation strategy (i.e., their augmentation subsets can be ordered, where the first includes two augmentations, the second the first two plus another one, etc.). Furthermore, they calculate the contrastive loss at different hierarchical levels in the siamese encoder---the authors motivate this step by conjecturing that this way invariances are imposed at different levels. Another key difference is that the embeddings of augmentation parameters are also included in the representation before the projection head (i.e., the loss has information about the data augmentation).

\subsection{The effect of normalization}

    \citet[Lem.~2]{tian_understanding_2022} shed light into why normalized representations are favorable: the author shows that the Frobenius norm of all weight matrices below an \ltwo- or layer normalization~\cite{ba2016layer} terms remains constant. The theory holds for nonlinear MLP with reversible activations, including ReLU; and, interestingly, provides a connection of why we could feasibly replace \ltwo-normalization with layer normalization.
    \citet{wen_mechanism_2022} explained dimensionality collapse for non-contrastive  SimSiam~\citep{chen_exploring_2020} with a two-layer nonlinear model and output normalization, and analyze the role of the projector. Empirically, when the projector is initialized to the identitiy matrix and only the off-diagonal entries are trainable, the authors obtain competitive representations. They also point out that batch normalization (by enforcing a nonzero variance) is less prone to collapse and conjecture that a complete collapse is more probable with \ltwo-normalization but do not prove the conjecture. Their results are in unison with ours since both works state that the content variables are learned first.

    \citet{ermolov_whitening_2021} proposed a non-contrastive SSL loss where an MSE objective is applied to whitened features and show empirical improvements regarding dimensionality collapse; furthermore, their empirical results suggest that using the InfoNCE loss with \ltwo-normalization and whitened features leads to divergence. Intuitively, there are some similarities between their normalization strategy and AnInfoNCE , though they are distinct as we normalize the latent factors before calculating the weighted MSE, whereas the authors whiten before normalization. Partially, this can explain why we do not have divergence in the case of \ltwo-normalization and weighting via the diagonal scaling matrix.

    \citet{halvagal_eigenspace_2022} proposed an isotropic loss to ``flatten out" the eigenvalue spectrum of the inferred latent covariance, which is remarkably similar to our strategy of using different temperature values.


\subsection{The role of negatives and hard negative sampling}\label{subsec:hn}
        
    \citet{arora_theoretical_2019} showed that more negative samples can hurt performance when the negative samples have the same class.
    \citet{chuang_debiased_2020} made the same observation, i.e.,  that sampling negative examples truly different (downstream) labels improves performance in a synthetic setting where labels are available.     
    They uncover a NPC in InfoNCE and propose to resolve it with debiasing, resulting in DCL that corrects for the sampling of same-label datapoints, even without knowledge of the true labels. They also bound the difference between the SimCLR and the debiased contrastive losses.
     \citet{robinson_contrastive_2021} extended \citet{chuang_debiased_2020} and propose a strategy to draw hard negative samples in an unsupervised way---using the data density and positive conditional distribution---for contrastive learning and analyze the resulting loss theoretically.
    The authors connect hard negative samples to adversarial samples and analyze the loss in the infinite sample limit in (c.f. their Thm.~4). They also bound the misclassification risk in their Thm.~5.
    \citet{liu_bayesian_2023} extended \citep{chuang_debiased_2020} into the Bayesian framework; the authors propose a hard negative sampling strategy with importance sampling.

    Several works tackle the problem of designing (hard) negative sampling strategies. \citet{wu_mutual_2020} generalized the negative sampling in InfoNCE, their method include more ``exotic" negative sampling schemes, e.g., from a ``ring" on a sphere.
    \citet{robinson_can_2021} proposed IFM as a solution that disturbs both positive and negative pairs, though in an implicit manner (it can be thought of as implicit hard negative sampling): it makes positive samples less, negative samples more similar (measured by cosine similarity) to the anchor point. This perturbation is done in the latent space  (c.f. Fig.~4 of their paper) and improves linear readout accuracy on multiple image data sets (Fig.~6; without the trade-offs induced by changing $\tau$) and also improves performance on medical data (Tab.~2). The authors also demonstrate that SimCLR combined with IFM improves robustness in the adversarial setting (
    Fig.~8).

    \citet{chuang_robust_2022} proposed RINCE, a symmetric drop-in replacement for InfoNCE, with a theoretical connection to robust symmetric losses in noisy binary classification. The loss (approximately) takes the $q^{th}$ exponent of the positive and negative loss terms, implementing a trade-off between exploration and exploitation. In the limit of $q=0$, it is the same as InfoNCE. Thus, the authors conclude, InfoNCE puts more weight on hard positives, whereas RINCE on easy positives. 
    RINCE can be thought as an implicit hard/easy positive mining scheme with an exponential weighting factor on uniformity.

\printglossaries